\documentclass[letterpaper]{article} 
\usepackage{pdfpages}
\usepackage{aaai23}  
\usepackage{times}  
\usepackage{helvet}  
\usepackage{courier}  
\usepackage[hyphens]{url}  
\usepackage{graphicx} 
\usepackage{natbib}  
\usepackage{caption} 
\usepackage{algorithm}
\usepackage{algorithmic}

\usepackage{newfloat}
\usepackage{listings}
\DeclareCaptionStyle{ruled}{labelfont=normalfont,labelsep=colon,strut=off} 
\floatstyle{ruled}
\newfloat{listing}{tb}{lst}{}
\floatname{listing}{Listing}
\usepackage{amsmath,amssymb,amsfonts}
\usepackage{graphicx}

\usepackage{bm} 
\usepackage{microtype}
\usepackage[utf8]{inputenc} 
\usepackage{url}            
\usepackage{nicefrac}       
\usepackage{microtype,mathrsfs}
\usepackage{subfigure}
\usepackage{caption}
\usepackage{breqn} 
\usepackage{enumitem}

\usepackage{cite}

\definecolor{darkred}{RGB}{230,0,0}
\definecolor{darkgreen}{RGB}{0,150,0}
\definecolor{darkblue}{RGB}{0,0,150}
\usepackage{amsthm}

\makeatletter
\newcommand*{\rom}[1]{\expandafter\@slowromancap\romannumeral #1@}
\makeatother

\newcommand{\mathleft}{\@fleqntrue\@mathmargin0pt}
\newcommand{\mathcenter}{\@fleqnfalse}

\makeatletter
\newcommand{\ssymbol}[1]{^{\@fnsymbol{#1}}}
\makeatother

\newcommand{\R}{\mathbb{R}}

\newtheorem{defn}{Definition}

\theoremstyle{theorem}
\newtheorem{ass}{Assumption}
\theoremstyle{remark}
\newtheorem{remark}{Remark}

\newcommand{\sign}{\texttt{sign}}

\newcommand{\E}{\mathbb{E}}                    

\newcommand{\nn}{\notag}

\newcommand{\beq}{\begin{equation}}
\newcommand{\eeq}{\end{equation}}
\newcommand{\bea}{\begin{align}}
\newcommand{\eea}{\end{align}}
\newcommand{\beas}{\begin{align*}}
\newcommand{\eeas}{\end{align*}}

\def\bea#1\eea{\begin{align}#1\end{align}}

\theoremstyle{plain}
\newtheorem{theorem}{Theorem}
\newtheorem{corollary}{Corollary}[theorem]
\newtheorem{lemma}[theorem]{Lemma}
\newtheorem{proposition}[theorem]{Proposition}
   \title{Fast Convergence in Learning Two-Layer Neural Networks \\with Separable Data}

\author {
    Hossein Taheri\textsuperscript{\rm 1},\;
    Christos Thrampoulidis\textsuperscript{\rm 2}
}
\affiliations {
   \textsuperscript{\rm 1}University of California, Santa Barbara\\
    \textsuperscript{\rm 2}University of British Columbia\\
    hossein@ucsb.edu, cthrampo@ece.ubc.ca
}

\begin{document}

\setcounter{secnumdepth}{2} 

\maketitle
\begin{abstract}
Normalized gradient descent has shown substantial success in speeding up the convergence of  exponentially-tailed loss functions (which includes exponential and logistic losses) on linear classifiers with separable data.  In this paper, we go beyond linear models by studying normalized GD on two-layer neural nets. We prove for exponentially-tailed losses that using normalized GD leads to linear rate of convergence of the training loss to the global optimum if the iterates find an interpolating model. This is made possible by showing certain gradient self-boundedness conditions and a log-Lipschitzness property. We also study generalization of normalized GD for convex objectives via an algorithmic-stability analysis. In particular, we show that normalized GD does not overfit during training by establishing finite-time generalization bounds. 
\end{abstract}
\section{Introduction}
\subsection{Motivation}
\par
A wide variety of machine learning algorithms for classification tasks rely on learning a model using monotonically decreasing loss functions such as logistic loss or exponential loss. In modern practice these tasks are often accomplished using over-parameterized models such as large neural networks where the model can interpolate the training data, i.e., it can achieve perfect classification accuracy on the samples.  In particular, it is often the case that the training of the model is continued until achieving approximately zero training loss \cite{zhang2021understanding}. 
\par
Over the last decade there has been remarkable progress in understanding or improving the convergence and generalization properties of over-parameterized models trained by various choices of loss functions including logistic loss and quadratic loss. For the quadratic loss it has been shown that over-parameterization can result in significant improvements in the training convergence rate of (stochastic)gradient descent on empirical risk minimization algorithms. Notably, quadratic loss on two-layer ReLU neural networks is shown to satisfy the Polyak-Łojasiewicz(PL) condition \cite{charles2018stability,bassily2018exponential,liu2022loss}. In fact, the PL property is a consequence of the observation that the tangent kernel associated with the model is a non-singular matrix. Moreover, in this case the PL parameter, which specifies the rate of convergence, is the smallest eigenvalue of the tangent kernel\cite{liu2022loss}. The fact that over-parameterized neural networks trained by quadratic loss satisfy the PL condition, guarantees that the loss convergences exponentially fast to a global optimum. The global optimum in this case is a model which ``perfectly'' interpolates the data, where we recall that perfect interpolation requires that the model output for every training input is precisely equal to the corresponding label.
\par
 On the other hand, gradient descent using un-regularized logistic regression with linear models and separable data is biased toward the max-margin solution. In particular, in this case the parameter converges in direction with the rate $O({1/\log(t)})$ to the solution of hard margin SVM problem, while the training loss converges to zero at the rate $\tilde{O}(1/t)$ \cite{soudry2018implicit,ji2018risk}. More recently, normalized gradient descent has been proposed as a promising approach for fast convergence of exponentially tailed losses. In this method, at any iteration the step-size is chosen proportionally to the inverse of value of training loss function \cite{nacson2019convergence}. This results in choosing unboundedly increasing step-sizes for the iterates of gradient descent. This choice of step-size leads to significantly faster rates for the parameter's directional convergence. In particular, for linear models with separable data, it is shown that normalized GD with decaying step-size enjoys a rate of $O(1/\sqrt{t})$ in directional parameter convergence to the max-margin separator \cite{nacson2019convergence}. This has been improved to $O(1/t)$ with normalized GD using fixed step-size \cite{ji2021characterizing}.
 \par
  Despite remarkable progress in understanding the behavior of normalized GD with separable data, these results are only applicable to the implicit bias behavior of ``linear models''. In this paper, we aim to discover for the first time, the dynamics of learning a two-layer neural network with normalized GD trained on separable data. We also wish to realize the iterate-wise test error performance of this procedure. We show that using normalized GD on an exponentially-tailed loss with a two layered neural network leads to exponentially fast convergence of the loss to the global optimum. This is comparable to the convergence rate of $O(1/t)$ for the global convergence of neural networks trained with exponentially-tailed losses. Compared to the convergence analysis of standard GD which is usually carried out using smoothness of the loss function, here for normalized GD we use the Taylor's expansion of the loss and use the fact the operator norm of the Hessian is bounded by the loss. Next, we apply a lemma in our proof which shows that exponentially-tailed losses on a two-layered neural network satisfy a log-Lipschitzness condition throughout the iterates of normalized GD. Moreover, crucial to our analysis is showing that the $\ell_2$ norm of the gradient at every point is upper-bounded and lower-bounded by constant factors of the loss under given assumptions on the activation function and the training data. Subsequently, the log-Lipschitzness property together with the bounds on the norm of Gradient and Hessian of the loss function ensures that normalized GD is indeed a descent algorithm. Moreover, it results in the fact that the loss value decreases by a constant factor after each step of normalized GD, resulting in the promised geometric rate of decay for the loss. 
\subsection{Contributions}

 In Section \ref{sec:caotl} we introduce conditions --namely log-Lipschitz and self-boundedness assumptions on the gradient and the Hessian-- under which the training loss of the normalized GD algorithm converges exponentially fast to the global optimum. More importantly, in Section \ref{sec:tlnn} we prove that the aforementioned conditions are indeed satisfied by two-layer neural networks trained with an exponentially-tailed loss function if the iterates lead to an interpolating solution. This yields the first theoretical guarantee on the convergence of normalized GD for non-linear models. We also study a stochastic variant of normalized GD and investigate its training loss convergence  in Section \ref{sec:stoch}.

In Section \ref{sec:gen-err} we study, for the first time, the finite-time test loss and test error performance of normalized GD for convex objectives. In particular, we provide sufficient conditions for the generalization of normalized GD and derive bounds of order $O(1/n)$ on the expected  generalization error, where $n$ is the training-set size. 

\subsection{Prior Works}
The theoretical study of the optimization landscape of over-parameterized models trained by GD or SGD has been the subject of several recent works. The majority of these works study over-parameterized models with specific choices of loss functions, mainly quadratic or logistic loss functions. For quadratic loss, the exponential convergence rate of over-parameterized neural networks is proved in several recent works e.g., \cite{charles2018stability,bassily2018exponential,du2019gradient,allen2019convergence,arora2019fine,oymak2019overparameterized,oymak2020toward,safran2021effects,liu2022loss}. These results naturally relate to the Neural Tangent Kernel(NTK) regime of infinitely wide or sufficiently large initialized neural networks \cite{jacot2018neural} in which the iterates of gradient descent stay close to the initialization. The NTK approach can not be applied to our setting as the parameters' norm in our setting is growing as $\Theta(t)$ with the NGD updates. 
\par 
The majority of the prior results apply to the quadratic loss. However, the state-of-the-art architectures for classification tasks use unregularized ERM with logistic/exponential loss functions. Notably, for these losses over-parameterization leads to infinite norm optimizers. As a result, the objective in this case does not satisfy strong convexity or the PL condition even for linear models. The analysis of loss and parameter convergence of logistic regression on separable data has attracted significant attention in the last five years. Notably, a line of influential works have shown that gradient descent provably converges in direction to the max-margin solution for linear models and two-layer homogenous neural networks. In particular, the study of training loss and implicit bias behavior of GD on logistic/exponential loss was first initiated in the settings of linear classifiers \cite{Rosset2003MarginML,telgarsky2013margins,soudry2018implicit,ji2018risk,nacson2019convergence}. The implicit bias behavior of GD with logistic loss in two-layer neural networks was later studied by \cite{lyu2019gradient,chizat2020implicit,ji2020directional}. The loss landscape of logistic loss for over-parameterized neural networks and structured data is analyzed in \cite{zou2020gradient,chatterji2021does}, where it is proved that GD converges to a global optima at the rate $O(1/t)$. The majority of these results hold for standard GD while we focus on normalized GD. 
\par
The generalization properties of GD/SGD with binary and multi-class logistic regression is studied in \cite{shamir2021gradient,schliserman2022stability} for linear models and in \cite{li2018learning,cao2019generalization,cao2020generalization} for neural networks. Recently, \cite{taheri2023generalization} studied the generalization error of decentralized logistic regression through a stability analysis. For our generalization analysis we use an algorithmic stability analysis \cite{bousquet2002stability,hardt2016train,lei2020fine}. However, unlike these prior works we consider normalized GD and derive the first generalization analysis for this algorithm. 
\par
The benefits of normalized GD for speeding up the directional convergence of GD for linear models was suggested by \cite{nacson2019convergence,ji2021characterizing}. Our paper contributes to this line of work. Compared to the prior works which are focused on implicit behavior of linear models, we study non-linear models and derive training loss convergence rates. We also study, the generalization performance of normalized GD for convex objectives. 
\subsection*{Notation}
We use $\|\cdot\|$ to denote the operator norm of a matrix and also to denote the $\ell_2$-norm of a vector. The Frobenius norm of a matrix $W$ is shown by $\|W\|_F$. The Gradient and the Hessian of a function $F:\R^{d}\rightarrow\R$ are denoted by $\nabla F$ and $\nabla^2 F$. Similarly, for a function $F:\R^{d}\times\R^{d'}\rightarrow\R$ that takes two input variables, the Gradient and the Hessian with respect to the $i$th variable (where $i=1,2$) are denoted by $\nabla_i F$ and $\nabla^2_i F$, respectively. For functions $F,G:\R\rightarrow\R$, we write $F(t)=O(G(t))$ when $|F(t)|\le m\,G(t)$ after $t\ge t_0$ for positive constants $m,t_0$. We write $F(t)=\tilde{O}(G(t))$ when $F(t)=O(G(t)H(t))$ for a polylogarithmic function $H$. Finally, we denote $F(t) = \Theta (G(t))$ if $|F(t)| \le m_1 G(t) $ and $|F(t)|\ge m_2 G(t)$ for all $t\ge t_0$ for some positive constants $m_1,m_2,t_0$.  

\subsection{Problem Setup}
We consider unconstrained and unregularized empirical risk minimization (ERM) on $n$ samples,
\bea\label{eq:logreg}
\min_{w\in\R^{\widetilde d}} F(w):=\frac{1}{n}\sum_{i=1}^{n} f\left(y_i \Phi(w,x_i)\right).
\eea
The $i$th sample $z_i:=(x_i,y_i)$ consists of a data point $x_i\in\R^d$ and its associated label $y_i\in\{\pm 1\}$.
The function $\Phi:\R^{\widetilde d} \times \R^d\rightarrow\R$ represents the model taking the weights vector $w$ and data point $x$ to approximate the label. In this section, we take $\Phi$ as a neural network with one hidden layer and $m$ neurons, 
\bea\nn
\Phi(w,x):= \sum_{j=1}^m a_j \sigma(\langle w_j,x\rangle). 
\eea
Here $\sigma:\R\rightarrow\R$ is the activation function and $w_j\in\R^d$ denotes the input weight vector of the $j$th hidden neuron. $w\in\R^{\widetilde d}$ represents the concatenation of these weights i.e., $w=[w_1;w_2;\dots;w_m]$. In our setting the total number of parameters and hence the dimension of $w$ is $\widetilde d=md.$
We assume that only the first layer weights $w_j$ are updated during training and the second layer weights $a_j\in\R$ are initialized randomly and are maintained fixed during training.
 The function $f:\R \rightarrow\R$ is non-negative and monotonically decreases such that $\lim_{t\rightarrow +\infty} f(t) = 0.$ In this section, we focus on the exponential loss $f(t)= \exp(-t)$, but we expect that our results apply to a broader class of loss functions that behave similarly to the exponential loss for large $t$, such as logistic loss $f(t)=\log(1+\exp(-t))$. 
 
We consider activation functions with bounded absolute value for the first and second derivatives. 
\begin{ass}[Activation function] \label{ass:act}The activation function $\sigma:\R\rightarrow\R$ is smooth and for all $t\in\R$
\bea\nn
 |\sigma^{\prime\prime}(t)| \le L.
\eea 
Moreover, there are positive constants $\alpha,\ell$ such that $\sigma$ satisfies for all $t\in\R$,
\bea\nn
\alpha \le \sigma^\prime(t) \le\ell.
\eea
\end{ass}
An example satisfying the above condition is the activation function known as smoothed-leaky-ReLU which is a smoothed variant of the leaky-ReLU activation $\sigma(t) = \ell t \, \mathbb{I} (t\ge0) + \alpha t\, \mathbb{I} (t\le0)$, where $\mathbb{I}(\cdot)$ denotes the 0--1 indicator function. 
\par 
Throughout the paper we let $R$ and $a$ denote the maximum norm of data points and second layer weights, respectively, i.e., 
\bea\nn
R:=\max_{i\in[n]} \;\;\|x_i\|\; ,\;\;\;\;\;\; a:=\max_{j\in[m]}\;\; |a_j|\;.
\eea
Throughout the paper we assume $R=\Theta(1)$ w.r.t. problem parameters and $a=\frac{1}{m}$.
\\
We also denote the \emph{training loss} of the model by $F$, defined in \eqref{eq:logreg} and define the \emph{train error} as misclassification error over the training data, or formally by $F_{0-1} (w):= \frac{1}{n}\sum_{i=1}^n \mathbb{I}(\textsc{sign}(\Phi(w,x_i))\neq y_i)$.

\paragraph*{Normalized GD.}We consider the iterates of normalized GD as follows, 
\bea\label{eq:dec_main}
w_{t+1} =w_t - \eta_t\nabla F(w_t). 
\eea
The step size is chosen inversely proportional to the loss value i.e., $\eta_t = \eta/F(w_t)$, implying that the step-size is growing unboundedly as the algorithm approaches the optimum solution. Since the gradient norm decays  proportionally to the loss, one can equivalently choose $\eta_t=\eta/\|\nabla F(w_t)\|$.
\section{Main Results}
For convergence analysis in our case study, we introduce a few definitions.
\begin{defn}[log-Lipschitz Objective]\label{ass:gencon}
The training loss $F:\R^{\widetilde d}\rightarrow\R$ satisfies the log-Lipschitzness property if for all $w,w'\in \R^{\tilde d}$,
\bea\nn
\max _{v\in [w,w']} F(v) \le F(w) \cdot \tilde{c}_{w,w'},
\eea
where $[w,w']$ denotes the line between $w$ and $w'$ and we define $\tilde{c}_{w,w'}:=\exp\left(c(\|w-w'\|+\|w-w'\|^2)\right)$ where the positive constant $c$ is independent of $w,w'$.
\end{defn}

As we will see in the following sections, log-Lipschitzness is a property of neural networks trained with exponentially tailed losses with $c=\Theta(\frac{1}{\sqrt{m}})$. We also define the property ``log-Lipschitzness in the gradient path'' if for all $w_t,w_{t-1}$ in Eq. \eqref{eq:dec_main} there exists a constant $C$ such that, 
\bea\nn
\max _{v\in[w_t,w_{t+1}]} F(v) \le C\,F(w_t). 
\eea

\begin{defn}[Self lower-bounded gradient]\label{ass:self-lower}
The loss function $F:\R^{\widetilde d}\rightarrow\R$ satisfies the self-lower bounded Gradient condition for a function, if these exists a constant $\mu$ such that for all $w$,
\bea\nn
\| \nabla F(w)\| \ge \mu\, F(w).
\eea
\end{defn}
\begin{defn}[Self-boundedness of the gradient]\label{ass:self-bounded}
The loss function $F:\R^{\widetilde d}\rightarrow\R$ satisfies the self-boundedness of the gradient condition for a constant $h$, if for all $w$ 
\bea\nn
\| \nabla F(w)\| \le h\, F(w).
\eea
\end{defn}

The above two conditions on the upper-bound and lower bound of the gradient norm based on loss can be thought as the equivalent properties of smoothness and the PL condition but for our studied case of exponential loss. To see this, note that smoothness and PL condition provide upper and lower bounds for the square norm of gradient. In particular, by $L$-smoothness one can deduce that $\|\nabla F(w)\|^2 \le 2L(F(w)-F^\star)$ (e.g., \cite{nesterov2003introductory}) and by the definition of $\mu$-PL condition $\|\nabla F(w)\|^2\ge 2\mu (F(w)-F^\star)$ \cite{Polyak1963GradientMF,Lojasiewicz}. 
\par
The next necessary condition is an upper-bound on the operator norm of the Hessian of loss. 
\begin{defn}[Self-boundedness of the Hessian]\label{ass:laplace}
The loss function $F:\R^{\widetilde d}\rightarrow\R$ satisfies the self-boundedness of the Hessian property for a constant $H$, if for all $w$, 
\bea\nn
\| \nabla^2 F(w)\| \le H\,F(w),
\eea
where $\|\cdot\|$ denotes the operator norm.
\end{defn}
It is worthwhile to mention here that in the next sections of the paper, we prove all the self lower and upper bound in Definitions \ref{ass:self-bounded}-\ref{ass:laplace} are satisfied for a two-layer neural network under some regularity conditions.
\subsection{Convergence Analysis of Training Loss} \label{sec:caotl}
 The following theorem states that under the conditions above, the training loss converges to zero at an exponentially fast rate. 

\begin{theorem}[Convergence of Training Loss]\label{lem:train} Consider normalized gradient descent update rule with loss $F$ and step-size $\eta_t$.  Assume $F$ and the normalized GD algorithm satisfy log-Lipschitzness in the gradient path with parameter $C$, as well as self-boundedness of the Gradient and the Hessian and the self-lower bounded Gradient properties with parameters $h,H$ and $\mu$, respectively. Let $\eta_t=\frac{\eta}{F(w_t)}$ for all $t\in[T]$ and for any positive constant $\eta$ satisfying $\eta \le \frac{\mu^2}{HCh^2}$. Then for the training loss at iteration $T$ the following bound holds:
\bea\label{eq:trainthm1}
F(w_{_T}) \le (1-\frac{\eta\mu^2}{2})^T F(w_0).
\eea
\end{theorem}
\begin{remark}
The proof of Theorem \ref{lem:train} is provided in Appendix \ref{sec:appA}, where we use a Taylor expansion of the loss and apply the conditions of the theorem. It is worth noting that the rate obtained for normalized GD in Theorem \ref{lem:train} is significantly faster than the rate of $\widetilde{O}(\frac{1}{T})$ for standard GD with logistic or exponential loss in neural networks (e.g., \cite[Thm 4.4]{zou2020gradient}, and \cite[Thm 2]{taheri2023generalizationNN}). Additionally, for a continuous-time perspective on the training convergence of normalized GD, we refer to Proposition \ref{propo:grad_flow} in the appendix, which presents a convergence analysis based on \emph{normalized Gradient Flow}. The advantage of this approach is that it does not require the self-bounded Hessian property and can be used to show exponential convergence of normalized Gradient Flow for leaky-ReLU activation.
\end{remark}

\subsection{Two-Layer Neural Networks}\label{sec:tlnn}

In this section, we prove that the conditions that led to Theorem \ref{lem:train} are in fact satisfied by a two-layer neural network. Consequently, this implies that the training loss bound in Eq.\eqref{eq:trainthm1} is valid for this class of functions. 
We choose $f(t)=\exp(-t)$ for simpler proofs, however an akin result holds for the broader class of exponentially tailed loss functions.   
\par
First, we start with verifying the log-Lipschitzness condition (Definition \ref{ass:gencon}). In particular, here we prove a variation of this property for the iterates of normalized GD i.e., where $w,w'$ are chosen as $w_t,w_{t+1}$. The proof is included in Appendix \ref{sec:appB}.
\begin{lemma}[log-Lipschitzness in the gradient path]\label{lem:relax}
Let $F$ be as in \eqref{eq:logreg} for the exponential loss $f$ and let $\Phi$ be a two-layer neural network with the activation function satisfying Assumption \ref{ass:act}. Consider the iterates of normalized GD with the step-size $\eta_t=\frac{\eta}{F(w_t)}$. Then for any $\lambda\in[0,1]$ the following inequality holds: 
\bea\label{eq:relax1}
F(w_t + \lambda (w_{t+1}-w_t)) \le \exp(\lambda c)\, F(w_t),
\eea
for a positive constant $c$ independent of $\lambda, w_t$ and $w_{t+1}$.
\\
As a direct consequence, it follows that,
\bea\label{eq:relax2}
\max _{v\in[w_t,w_{t+1}]} F(v) \le C\,F(w_t),
\eea
for a numerical constant $C$.
\end{lemma}
\par
The next two lemmas state sufficient conditions for $F$ to satisfy the self-lower boundedness for its gradient (Definition \ref{ass:self-lower}). The proofs are deferred to Appendices \ref{sec:appC}-\ref{sec:appD}. 
\begin{lemma}[Self lower-boundedness of gradient]\label{lem:low}
Let $F$ be as in \eqref{eq:logreg} for the exponential loss $f$ and let $\Phi$ be a two-layer neural network with the activation function satisfying Assumption \ref{ass:act}. Assume the training data is linearly separable with margin $\gamma$. Then $F$ satisfies the self-lower boundedness of gradient with the constant $\mu= \frac{\alpha \gamma}{\sqrt{m}}$ for all $w$, i.e., $\|\nabla F(w)\| \ge \mu F(w).$
\end{lemma}

Next, we aim to show that the condition $\|\nabla F(w)\|\ge \mu F(w)$, holds for training data separable by a two-layer neural network during gradient descent updates. In particular, we assume the Leaky-ReLU activation function taking the following form,
\bea\label{eq:piecewisel}
\sigma(t) = \begin{cases}\ell \,t ~~~t\ge 0,\\
\alpha\, t~~~ t<0.
\end{cases}
\eea
for arbitrary non-negative constants $\alpha,\ell$. This includes the widely-used ReLU activation as a special case. Next lemma shows that when the weights are such that the neural network separates the training data, the self-lower boundedness condition holds. 

\begin{lemma}\label{lem:low2}
Let $F$ be in \eqref{eq:logreg} for the exponential loss $f$ and let $\Phi$ be a two-layer neural network with activation function in Eq.\eqref{eq:piecewisel}. Assume the first layer weights $w\in\R^{\widetilde d}$ are such that the neural network separates the training data with margin $\gamma$. Then $F$ satisfies the self- lower boundedness of gradient, i.e,
$\|\nabla F(w)\| \ge \mu F(w),$ where $\mu= \gamma$.
\end{lemma}

A few remarks are in place. The result of Lemma \ref{lem:low2} is relevant for $w$ that can separate the training data. Especially, this implies the self lower-boundedness property after GD iterates succeed in finding an interpolator. However, we should also point out that the non-smoothness of leaky-ReLU activation functions precludes the self-bounded Hessian property and it remains an interesting future direction to prove the self lower-boundedness property with general smooth activations.  On the other hand, the convergence of normalized "Gradient-flow" does not require the self-bounded Hessian property, as demonstrated in Proposition \ref{propo:grad_flow}. This suggests that Lemma \ref{lem:low2} can be applied to prove the convergence of normalized Gradient-flow with leaky-ReLU activations. It is worth highlighting that we have not imposed any specific initialization conditions in our analysis as the self-lower bounded property is essentially sufficient to ensure global convergence. 
\par
Next lemma derives the self-boundedness of the gradient and Hessian (c.f. Definitions \ref{ass:self-bounded}-\ref{ass:laplace}) for our studied case. The proof of Lemma \ref{lem:up-low} (in Appendix \ref{sec:appE}) follows rather straight-forwardly from the closed-form expressions of gradient and Hessian and using properties of the activation function. 

\begin{lemma}[Self-boundedness of the gradient and Hessian]\label{lem:up-low}
Let $F$ be in \eqref{eq:logreg}  for the exponential loss $f$ and let $\Phi$ be a two-layer neural network with the activation function satisfying Assumption \ref{ass:act}. Then $F$ satisfies the self-boundedness of gradient and Hessian with constants $h= \frac{\ell R}{\sqrt{m}},H:= \frac{LR^2}{m^2} + \frac{\ell^2 R^2}{m}$ i.e.,
\bea
 \|\nabla F(w)\| \le h F(w),\;\;\;\;  \|\nabla^2 F(w)\| \le H F(w).\nn
\eea
\end{lemma}

We conclude this section by offering a few remarks regarding our training convergence results. We emphasize that combining Theorem \ref{lem:train} and Lemmas \ref{lem:relax}-\ref{lem:up-low} achieves the convergence of training loss of normalized Gradient Descent for two-layer networks. Moreover, in Appendix \ref{sec:appJ}, we refer to Proposition \ref{propo:grad_flow} which presents a continuous time convergence analysis of normalized GD based on Gradient Flow. This result is especially relevant in the context of leaky-ReLU activation, where Proposition \ref{propo:grad_flow} together with Lemma \ref{lem:low2} shows exponential convergence of normalized Gradient-flow. The experiments of the training performance of normalized GD are deferred to Section \ref{sec:numeric}. 

\subsection{Generalization Error}\label{sec:gen-err}
In this section, we study the generalization performance of normalized GD algorithm.  
Formally, the \emph{test loss} for the data distribution $\mathcal{D}$ is defined as follows, 
\bea\nn
\widetilde F(w):=\E_{(x,y)\sim\mathcal{D}}\,\Big[f(y \Phi(w,x))\Big].
\eea
Depending on the choice of loss $f$, the test loss might not always represent correctly the classification performance of a model. For this, a more reliable standard is the \emph{test error} which is based on the $0-1$ loss,
\bea\nn
\widetilde {F}_{0-1}(w):= \E_{(x,y)\sim\mathcal{D}}\,\Big[\mathbb{I}(y \neq \textsc{sign}(\Phi(w,x)))\Big].
\eea
We also define the \emph{generalization loss} as the gap between training loss and test loss. Likewise, we define the \emph{generalization error} based on the train and test errors. 
\par
With these definitions in place, we are ready to state our results. In particular, in this section we prove that under the normalized GD update rule, the generalization loss at step $T$ is bounded by $O(\frac{T}{n})$ where recall that $n$ is the training sample size. While, the dependence of generalization loss on $T$ seems unappealing, we show that this is entirely due to the fact that a convex-relaxation of the $0-1$ loss, i.e. the loss function $f$,  is used for evaluating the generalization loss. In particular, we can deduce that under appropriate conditions on loss function and data (c.f. Corollary \ref{lem:teste}), the test error is related to the test loss through,
$$
\widetilde F_{0-1}(w_{_T}) = O(\frac{ \widetilde F(w_{_T})}{\|w_{_T}\|}).
$$ 
As we will see in the proof of Corollary \ref{lem:teste}, for normalized GD with exponentially tailed losses the weights norm $\|w_{_T}\|$ grows linearly with $T$. Thus, this relation implies that the test error satisfies $\widetilde F_{0-1}(w_{_T})=O(\frac{1}{n})$. Essentially, this bound on the misclassification error signifies the fast convergence of normalized GD on test error and moreover, it shows that normalized GD never overfits during its iterations.  

 It is worthwhile to mention that our generalization analysis is valid for any model $\Phi$ such that $f(y\Phi(\cdot,x))$ is convex for any $(x,y)\sim\mathcal{D}$. This includes linear models i.e., $\Phi(w,x) = \langle w,x\rangle$ or the Random Features model \cite{NIPS2007_013a006f}, i.e., $\Phi(w,x) = \langle w,\sigma(A x)\rangle$ where $\sigma(\cdot)$ is applied element-wise on its entries and the  matrix $A\in\R^{m\times d}$ is initialized randomly and kept fixed during train and test time. Our results also apply to neural networks in the NTK regime due to the convex-like behavior of optimization landscape in the infinite-width limit. 

We study the generalization performance of normalized GD, through a stability analysis \cite{bousquet2002stability}. The existing analyses in the literature for algorithmic stability of $\tilde L-$smooth losses, rely on the step-size satisfying $\eta_t = O(1/\tilde L)$. This implies that such analyses can not be employed for studying increasingly large step-sizes as in our case $\eta_t$ is unboundedly growing. 
In particular, the common approach in the stability analysis \cite{hardt2016train,lei2020fine} uses the ``non-expansiveness'' property of standard GD with smooth and convex losses, by showing that for $\eta\le 2/\tilde L$ and for any two points $w,v\in\R^{\tilde d}$, it holds that
$
\|w-\eta\nabla F(w) - (v-\eta\nabla F(v))\|\le \|w-v\|.
$
Central to our stability analysis is showing that under the assumptions of self-boundedness of Gradient and Hessian,  the  normalized GD update rule satisfies the non-expansiveness condition with any step-size satisfying both $\eta\lesssim\frac{1}{F(w)}$ and $\eta\lesssim\frac{1}{F(v)}$. The proof is included in Appendix \ref{sec:appI}. 

\begin{lemma}[Non-expansiveness of normalized GD]\label{lem:expansive}
Assume the loss $F$ to satisfy convexity and self-boundedness for the gradient and the Hessian with parameter $h\le 1$ (Definitions \ref{ass:self-bounded}-\ref{ass:laplace}). Let $v,w\in\R^d$. If $\eta \le \frac{1}{h \cdot\max(F(v),F(w))}$, then
\bea\nn
\|w-\eta\nabla F(w) - (v-\eta\nabla F(v))\| \le \|w-v\|.
\eea
\end{lemma}
\par
The next theorem characterizes the test loss for both Lipschitz and smooth objectives. Before stating the theorem, we need to define $\delta$. For the leave-one-out parameter $w_t^{\neg i}$ and loss $F^{\neg i}(\cdot)$ defined as
\bea\nn
w_{t+1}^{\neg i} = w_t^{\neg i} - \eta_t \nabla F^{\neg i}(w_t^{\neg i}),  
\eea
and 
\bea
F^{\neg i}(w): = \frac{1}{n} \sum_{\substack{{j=1}\\ {j\neq i}}}^n f(w,z_j), \nn
\eea
we define $\delta\ge1$ to be any constant which satisfies for all $t\in[T],i\in[n]$, the following
$$ F^{\neg i}(w_t^{\neg i})\le \delta\, F^{\neg i}(w_t).$$
While this condition seems rather restrictive, we prove in Lemma \ref{lem:onthm8} in Appendix \ref{sec:appc3} that the condition on 
$\delta$ is satisfied by two-layer neural networks with sufficient over-parameterization. With these definitions in place, we are ready to state the main theorem of this section.
\begin{theorem}[Test loss]\label{lem:test}
Consider normalized GD update rule with $\eta_t=\frac{\eta}{F(w_t)}$ where $\eta\le \frac{1}{h \delta}$. Assume the loss $F$ to be convex and to satisfy the self-bounded gradient and Hessian property with a parameter $h$ (Definitions \ref{ass:self-bounded}-\ref{ass:laplace}). Then the following statements hold for the test loss:  
\\
\\
(i) if the loss $F$ is $G$-Lipschitz, then the generalization loss at step $T$ satisfies

\bea
\E [\widetilde F(w_{_T})-F(w_{_T})]\le \frac{2GT}{ n}.\nn
\eea
\\
(ii) if the loss $F$ is $\tilde L$-smooth, then the test loss at step $T$ satisfies, 

\bea
\E[\widetilde F(w_{_T})]\le 4\E[F(w_{_T})] +\frac{3\tilde L^2 T}{ n},\nn
\eea
where all expectations are over training sets.
\end{theorem}

The proof of Theorem \ref{lem:test} is deferred to Appendix \ref{sec:appG}. As discussed earlier in this section, the test loss dependence on $T$ is due to the rapid growth of the $\ell_2$ norm of $w_t$. As a corollary, we show that the generalization error is bounded by $O(\frac{1}{n})$. For this, we assume the next condition.
\begin{ass}[Margin]\label{ass:margin}
There exists a constant $\tilde \gamma$ such that after sufficient iterations the model satisfies $|\Phi(w_t,x)|\ge  \tilde \gamma \|w_t\|$ almost surely over the data distribution $(x,y)\sim\mathcal{D}$.  
\end{ass}
Assumption \ref{ass:margin} implies that the absolute value of the margin is $\tilde \gamma$ i.e., $\frac{|\Phi(w_t,x)|}{\|w_t\|} \ge \tilde\gamma $ for almost every $x$ after sufficient iterations. This assumption is rather mild, as intuitively it requires that data distribution is not concentrating around the decision boundaries.   
\par

For the loss function, we consider the special case of logistic loss $f(t)=\log(1+\exp(-t))$ for simplicity of exposition and more importantly due to its Lipschitz property. The use of Lipschitz property is essential in view of Theorem \ref{lem:test}. 
\begin{corollary}[Test error]\label{lem:teste}
 Suppose the assumptions of Theorem \ref{lem:test} hold. Consider the neural network setup under Assumptions \ref{ass:act} and \ref{ass:margin} and let the loss function $f$ be the logistic loss. Then the  test error at step $T$ of normalized GD satisfies the following:
\bea\nn
\E [\widetilde F_{0-1}(w_{_T}) ]= O(\frac{1}{T} \E[F(w_{_T})] + \frac{1}{n})
\eea
%
\end{corollary}

The proof of Corollary \ref{lem:teste} is provided in Appendix \ref{sec:appH}. In the proof, we use that $\|w_t\|$ grows linearly with $t$ as well as Assumption \ref{ass:margin} to deduce $\widetilde F_{0-1}(w_{_T}) = O(\frac{ \widetilde F(w_{_T})}{T}).$ Hence, the statement of the corollary follows from Theorem \ref{lem:test} (i). We note that while we stated the corollary for the neural net setup, the result is still valid for any model $\Phi$ that satisfies the Lipschitz property in $w$. We also note that the above result shows the $\frac{1}{n}$-rate for expected test loss which is known to be optimal in the realizable setting we consider throughout the paper. 


\subsection{Stochastic Normalized GD}\label{sec:stoch}
In this section we consider a stochastic variant of normalized GD algorithm, Assume $z_t$ to be the batch selected randomly from the dataset at iteration $t$. The stochastic normalized GD takes the form,
\bea\label{eq:stocngd}
w_{t+1} = w_t - \eta_t \nabla F_{z_t}(w_t),
\eea
where $\nabla F_{z_t}(w_t)$ is the gradient of loss at $w_t$ by using the batch of training points $z_t$ at iteration $t$. We assume $\eta_t$ to be proportional to $1/F(w_t)$. Our result in this section states that under the following strong growth condition \cite{schmidt2013fast,vaswani2019fast}, the training loss converges at an exponential rate to the global optimum. 
\begin{ass}[Strong Growth Condition]\label{ass:SGC}
The training loss $F:\R^{\widetilde d}\rightarrow\R$ satisfies the strong growth condition with a parameter $\rho$,
\bea\nn
\E_{z}[\|\nabla F_z(w)\|^2]\le \rho \|\nabla F(w)\|^2.
\eea
\end{ass}
Notably, we show in Appendix \ref{sec:appK} that the strong growth condition holds for our studied case under the self-bounded and self-lower bounded gradient property.
\par
The next theorem characterizes the rate of decay for the training loss. The proof and numerical experiments are deferred to Appendices \ref{sec:appe2} and \ref{sec:appnum}, respectively. 
\begin{theorem}[Convergence of Training Loss]\label{thm:SNGD}
Consider stochastic normalized GD update rule in Eq.\eqref{eq:stocngd}. Assume $F$ satisfies Assumption \ref{ass:SGC} as well as the log-Lipschitzness in the GD path, self-boundedness of the Gradient and the Hessian and the self-lower bounded Gradient properties (Definitions \ref{ass:gencon}-\ref{ass:laplace}). Let $\eta_t=\eta/F(w_t)$ for all $t\in[T]$ and for any positive constant $\eta$ satisfying $\eta \le \frac{\mu^2}{HC\rho h^2}$. Then for the training loss at iteration $T$ the following bound holds:
\bea\nn
F(w_{_T}) \le (1-\frac{\eta\mu^2}{2})^T F(w_0).
\eea
\end{theorem}
\section{Numerical Experiments}\label{sec:numeric}

\begin{figure*}[h]
\centering
\includegraphics[width=5.1cm,height=3.8cm]{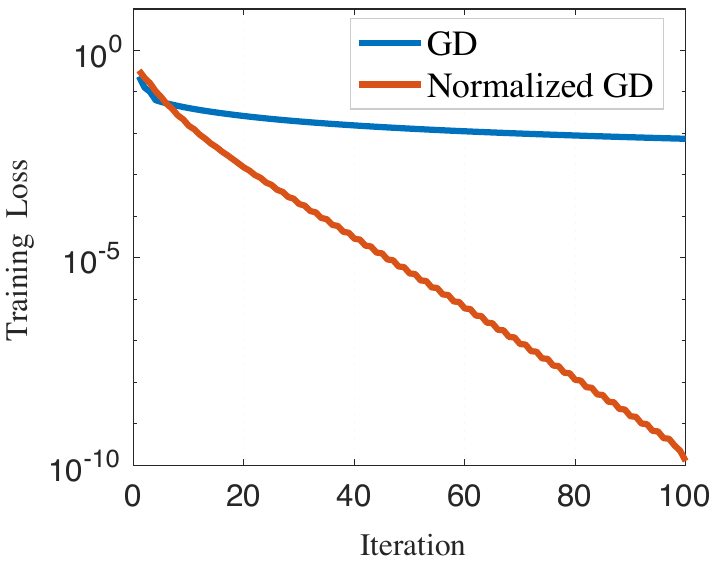}
\;\;
\includegraphics[width=5.1cm,height=3.8cm]{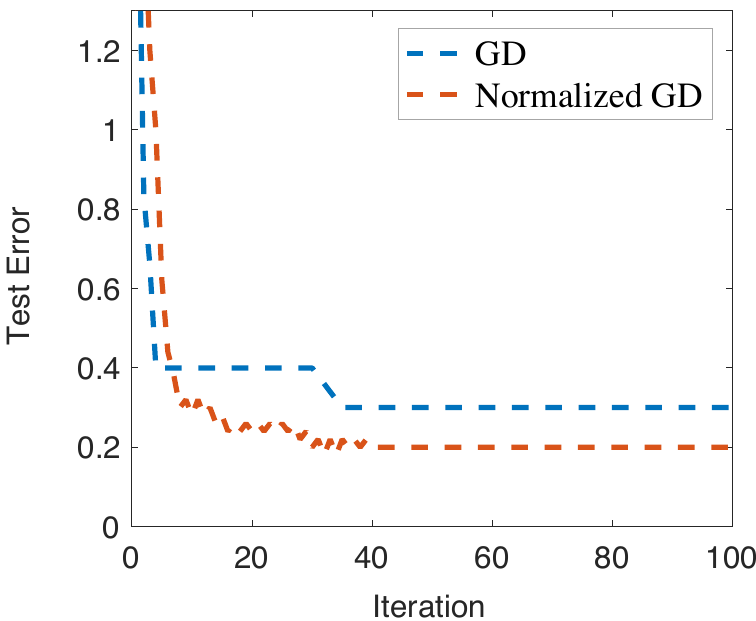}
\;\;
\includegraphics[width=5.1cm,height=3.8cm]{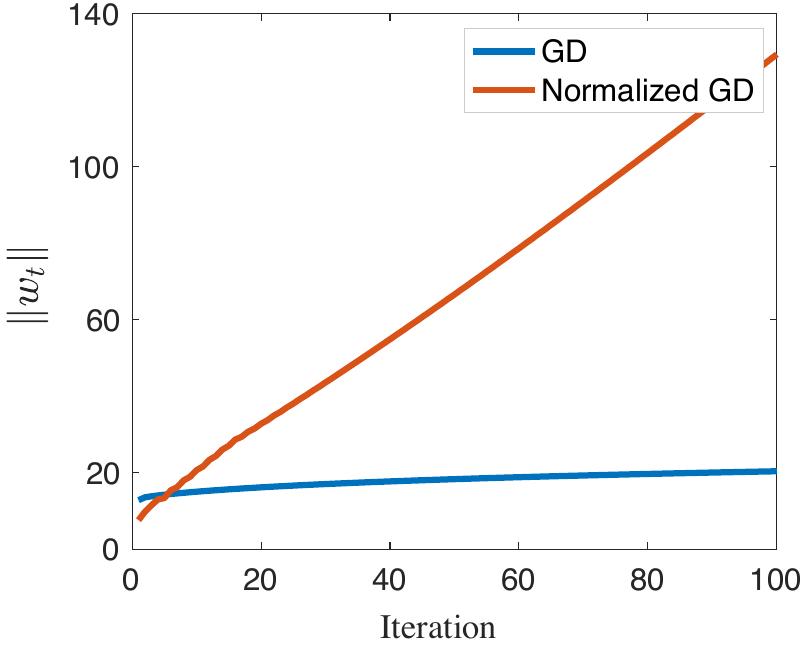}

\caption{Comparison of the training loss, test error (in percentage), and weight norm (i.e., $\|w_t\|$) between gradient descent and normalized gradient descent algorithms. The experiments were conducted on two classes of the MNIST dataset using exponential loss and a two-layer neural network with $m=50$ hidden neurons. The results demonstrate the performance advantages of normalized gradient descent over traditional gradient descent in terms of both the training loss and test error. }
\label{fig:1}
\end{figure*}
\begin{figure*}[h]
\centering
\includegraphics[width=4.8cm,height=3.7cm]{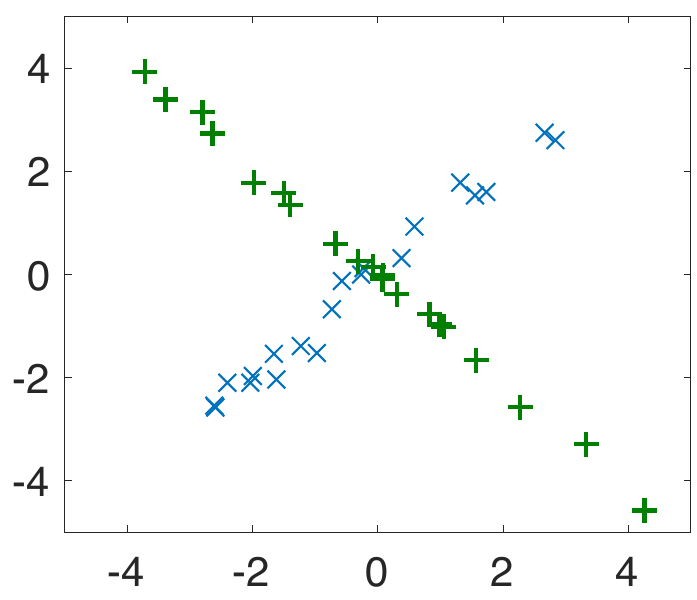}
\hspace{1in}
\includegraphics[width=5.3cm,height=3.7cm]{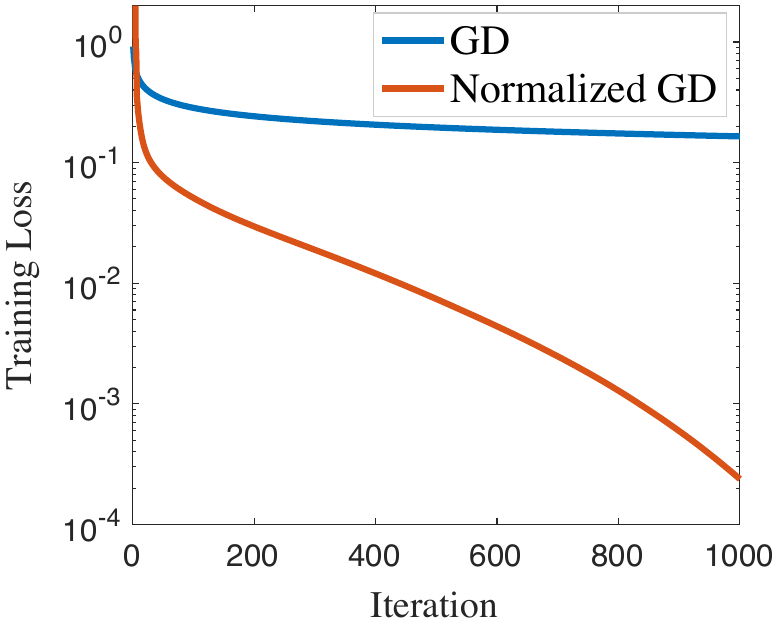}
\\[7pt]
\;\;\includegraphics[width=4.8cm,height=3.7cm]{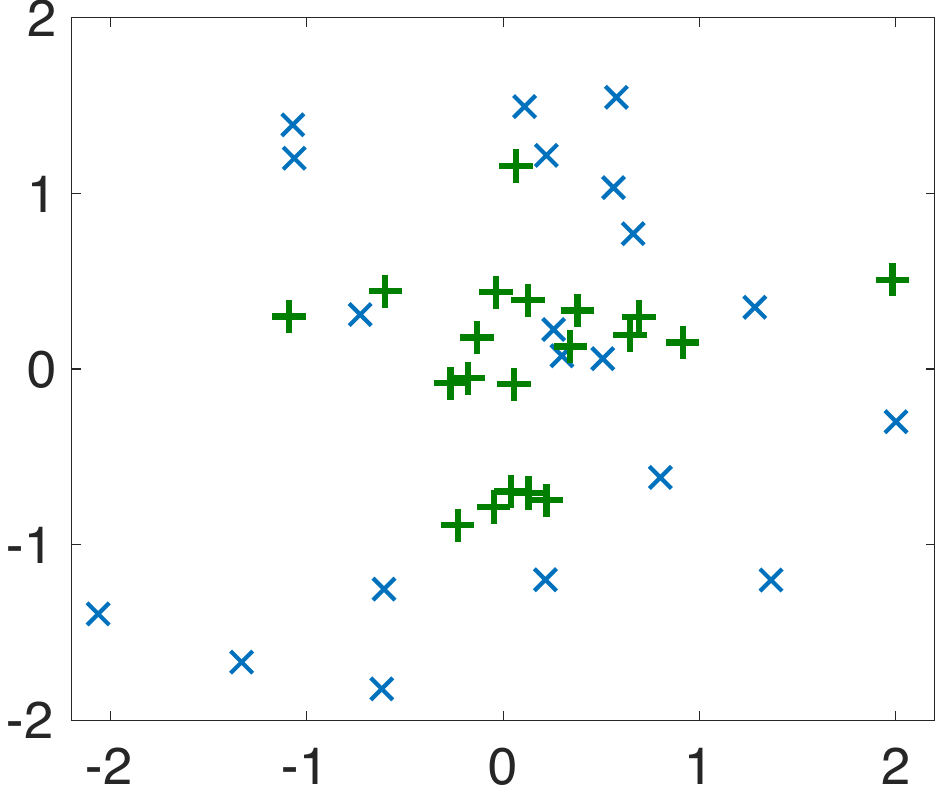}
\hspace{0.9in}
\includegraphics[width=5.6cm,height=3.8cm]{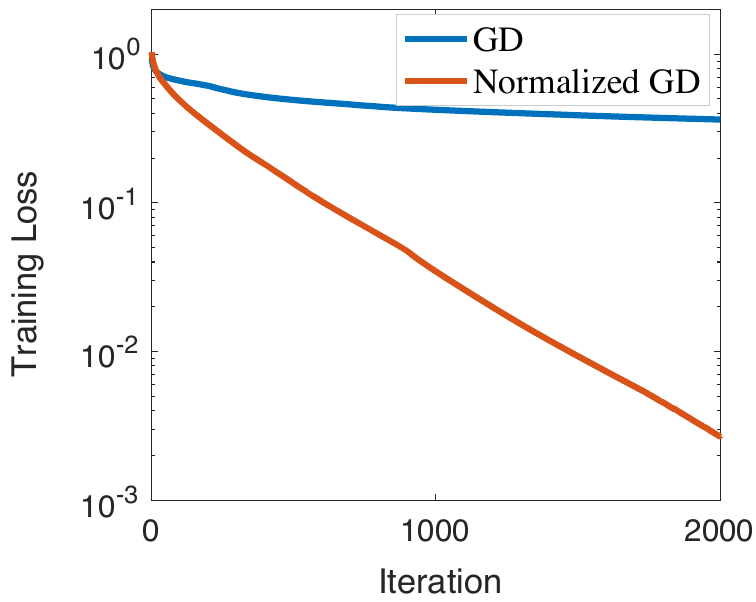}
\caption{The left plot depicts two synthetic datasets, each consisting of $n=40$ data points. On the right, we present the training loss results of gradient descent and normalized gradient descent algorithms applied to a two-layer neural network with $m=50$ (top) and $100$ (bottom) hidden neurons.}
\label{fig:2}
\end{figure*}
In this section, we demonstrate the empirical performance of normalized GD. It is important to highlight that the advantages of normalized GD over standard GD are most pronounced when dealing with well-separated data, such as in high-dimensional datasets. However, in scenarios where the margin is small, the benefits of normalized GD may be negligible.
Figure \ref{fig:1} illustrates the training loss (Left), the test error \% (middle), and the weight norm (Right) of GD with normalized GD. The experiments are conducted on a two-layer neural network with $m=50$ hidden neurons with leaky-ReLU activation function in \eqref{eq:piecewisel} where $\alpha=0.2$ and $\ell=1$. The second layer weights are chosen randomly from $a_j\in\{\pm \frac{1}{m}\}$ and kept fixed during training and test time. The first layer weights are initialized from standard Gaussian distribution and then normalized to unit norm. We consider binary classification with exponential loss using  digits $``0"$ and $``1"$ from the MNIST dataset ($d=784$) and we set the sample size to $n=1000$. The step-size are fine-tuned to $\eta=30$ and $5$ for GD and normalized GD, respectively so that each line represents the best of each algorithm. We highlight the significant speed-up in the convergence of normalized GD compared to standard GD. For the training loss, normalized GD decays exponentially fast to zero while GD converges at a remarkably slower rate.  We also highlight that $\|w_t\|$ for normalized GD grows at a rate $\Theta (t)$ while it remains almost constant for GD. In fact this was predicted by Corollary \ref{lem:teste} where in the proof we showed that the weight norm grows linearly with the iteration number.
 In Figure \ref{fig:2}, we generate two synthetic dataset according to a realization of a zero-mean Gaussian-mixture model with $n-40$ and $d=2$  where the two classes have different covariance matrices (top) and a zero-mean Gaussian-mixture model with $n=40,d=5$ (only the first two entires are depicted in the figure) where $\Sigma_1 = \mathbf{I}, \Sigma_2=\frac{1}{4}\mathbf{I}$ (Bottom). Note that none of the datasets is linearly separable. We consider the same settings as in Figure \ref{fig:1} and compared the performance of GD and normalized GD in the right plots. The step-sizes are fine-tuned to $\eta=80,350$ and $30,20$ for GD and normalized GD, respectively.  Here again the normalized GD algorithm demonstrates a superior rate in convergence to the final solution. 
 \section{Conclusions}
We presented the first theoretical evidence for the convergence of normalized gradient methods in non-linear models. While previous results on standard GD for two-layer neural networks trained with logistic/exponential loss proved a rate of $\widetilde{O}(1/t)$ for the training loss, we showed that normalized GD enjoys an exponential rate. We also studied for the first time, the stability of normalized GD and derived bounds on its generalization performance for convex objectives. We also briefly discussed the stochastic normalized GD algorithm. As future directions, we believe extensions of our results to deep neural networks is interesting. Notably, we expect several of our results to be still true for deep neural networks. Extending the self lower-boundedness property in Lemma \ref{lem:low2} for smooth activation functions is another important direction. Another promising avenue for future research is the derivation of generalization bounds for non-convex objectives by extending the approach used for GD (in \cite{taheri2023generalizationNN}) to normalized GD.
\section*{Acknowledgements}
This work was partially supported by NSF under Grant CCF-2009030. 
\bibliography{main}
\appendix
\onecolumn
\section*{Appendix}
\section {Proof of Theorem \ref{lem:train}}\label{sec:appA}
Based on the conditions of the theorem we have,
\bea
\max _{v\in[w_t,w_{t+1}]} F(v) &\le C\, F(w_t),\nn\\
 \|\nabla ^2 F(w)\| &\le H F(w) \;\;\;\text{and}\;\;\; \|\nabla F(w)\|\in [\mu F(w),h F(w)]\nn
 \eea
Then by Taylor's expansion and using the assumptions of the theorem we can deduce,
\begin{align*}
F(w_{t+1}) &\le  F(w_t) + \langle \nabla F(w_t),w_{t+1}-w_t\rangle + \frac{1}{2}\max_{v\in[w_t,w_{t+1}]} \|\nabla ^2 F(v)\|\cdot\|w_{t+1}-w_t\|^2   \\[4pt]
&\le F(w_t) -\eta_t \|\nabla F(w_t)\|^2 + \frac{\eta_t^2}{2}\max_{v\in[w_t,w_{t+1}]} \|\nabla ^2 F(v)\|\cdot \|\nabla F(w_t)\|^2\\[4pt]
&\le F(w_t) -\eta_t \|\nabla F(w_t)\|^2 + \frac{\eta_t^2 H}{2}\max_{v\in[w_t,w_{t+1}]} F(v)\cdot \|\nabla F(w_t)\|^2\\[4pt]
&\le F(w_t) -\mu^2\eta_t (F(w_t))^2 + \frac{\eta_t^2 H C h^2}{2} (F(w_t))^3
\end{align*}
Let $\eta_t = \frac{\eta}{F(w_t)}$,
\bea\nn
F(w_{t+1}) &\le (1-\eta \mu^2 + \frac{HCh^2\eta^2}{2})F(w_t)
\eea
Then condition on the step-size $\eta\le \frac{\mu^2}{HCh^2}$, ensures that $1-\eta \mu^2 + \frac{HCh^2\eta^2}{2} \le 1-\frac{\eta\mu^2}{2}$. Thus,
\bea\nn
F(w_{t+1}) \le (1-\frac{\eta\mu^2}{2})F(w_t) .
\eea
Thus $F(w_{_T}) \le (1-\frac{\eta\mu^2}{2})^T F(w_0).$ This completes the proof.

\section{Proofs for Section \ref{sec:tlnn}}
\subsection{Proof of Lemma \ref{lem:relax}}\label{sec:appB}

For a sample point $x\in\R^d$ and two weight vectors $w,w'\in\R^{\widetilde d}$, since the activation function satisfies $\sigma'<\ell,\sigma''<L$, we can deduce that, 
\bea\nn
|\Phi(w,x)-\Phi(w',x)| &= |\sum_{j=1}^m a_j \sigma(\langle w_j,x\rangle) - a_j\sigma(\langle w_j',x\rangle)|\\
&\le \sum_{j=1}^m |a_j|\cdot |\sigma(\langle w_j,x\rangle) -\sigma(\langle w_j',x\rangle)|\nn
\eea
By $L$-smoothness of the activation function and recalling that $\sigma'(\cdot)\le\ell$ we can write,
\bea
\sigma(\langle w_j,x\rangle) -\sigma(\langle w_j',x\rangle) &\le \sigma'(\langle w_j',x\rangle) \langle w_j-w_j',x\rangle+ \frac{L}{2}|\langle w_j-w_j',x\rangle|^2\nn\\[4pt]
&\le|\sigma'(\langle w_j',x\rangle)| \cdot |\langle w_j-w_j',x\rangle|+ \frac{L}{2}|\langle w_j-w_j',x\rangle|^2\nn\\[4pt]
&\le \ell \|w_j-w_j'\| \|x\|+ \frac{L}{2}\|w_j-w_j'\|^2\|x\|^2\nn\\[4pt]
&\le \ell R \|w_j-w_j'\|+ \frac{LR^2}{2}\| w_j-w_j'\|^2.\nn
\eea
Since by assumption $|a_j| \le a$,
\bea
|\Phi(w,x)-\Phi(w',x)| &\le \sum_{j=1}^m |a_j| (\ell R \|w_j-w_j'\|+ \frac{LR^2}{2}\| w_j-w_j'\|^2)\nn\\
&\le a R \sum_{j=1}^m (\ell  \|w_j-w_j'\| + LR \|w_j-w_j'\|^2).\nn
\eea
Hence, for a label $y\in\{\pm 1\}$ we have
\bea
-y\Phi(w,x)+y\Phi(w',x) &\le |\Phi(w,x)-\Phi(w',x)|\nn \\
&\le  a R \sum_{j=1}^m (\ell  \|w_j-w_j'\| + LR \|w_j-w_j'\|^2).\nn
\eea
Noting the use of exponential loss and by taking $\exp(\cdot)$ of both sides,

\bea
\frac{f(y\Phi(w,x))}{f(y\Phi(w',x))} &= \exp \left(-y\Phi(w,x)+y\Phi(w',x)\right)\nn\\
&\le \exp\Big( a R \sum_{j=1}^m (\ell  \|w_j-w_j'\| + LR \|w_j-w_j'\|^2)\Big)\nn\\
&\le \exp\left( a R (\sqrt{m}\, \ell \| w-w'\| + LR\|w-w'\|^2) \right)\label{eq:conv-sem}
\eea

Thus for any two points $w,w'$ it holds,
\bea\label{eq:conv-sem}
f(y\Phi(w,x)) \le  f(y\Phi(w',x)) \cdot \exp\left( a R(\sqrt{m}\, \ell \| w-w'\| + LR\|w-w'\|^2) \right)
\eea
Therefore, for a sample loss with $(x_i,y_i)\in\R^d\times\{\pm 1\}$ and $v\in[w_t,w_{t+1}]$ i.e, $v= w_t + \lambda (w_{t+1}-w_t)$ for some $\lambda \in [0,1]$, we have,
\bea
f(y_i\Phi(v,x_i)) &= f(y_i\Phi(w_t + \lambda (w_{t+1}-w_t),x_i))\nn\\[4pt]
&\le f(y_i \Phi(w_t.x_i)) \cdot \exp\left( a R  (\sqrt{m}\,\ell  \|v-w_t\| + LR \|v-w_t\|^2)\right)\nn\\[5pt]
&=  f(y_i \Phi(w_t.x_i)) \cdot \exp\left( a R (\sqrt{m}\,\ell \lambda \|w_{t+1}-w_t\| + LR \lambda^2\|w_{t+1}-w_t\|^2)\right)\nn\\[5pt]
&= f(y_i \Phi(w_t.x_i)) \cdot \exp\left( a R  (\sqrt{m}\,\ell \lambda \eta_t\|\nabla F(w_t)\| + LR \lambda^2\eta_t^2\|\nabla F(w_t)\|^2)\right)\nn\\[4pt]
&=f(y_i \Phi(w_t.x_i)) \cdot \exp\Big( a R  (\sqrt{m}\,\ell \lambda \frac{\eta}{F(w_t)}\|\nabla F(w_t)\| + LR \lambda^2(\frac{\eta}{F(w_t)})^2\|\nabla F(w_t)\|^2)\Big)\nn\\[3pt]
&\le f(y_i \Phi(w_t.x_i)) \cdot \exp\left( \sqrt{m}\,aR\,\ell \lambda h\eta + aLR^2\lambda^2h^2\eta^2\right),\nn
\eea
where for the last step we used the assumption that $\eta_t=\frac{\eta}{F(w_t)}$ for any constant $\eta\le \frac{\mu^2}{HCh^2}$ and the assumption that $\|\nabla F(w)\|\le h F(w)$.
This proves the inequality \eqref{eq:relax1} in the statement of the lemma. 

To derive \eqref{eq:relax2}, note that since $\lambda\le 1$,
\bea
\max _{v\in[w_t,w_{t+1}]} f(y_i\Phi(v,x_i)) &= \max_{\lambda\in[0,1]} f(y_i\Phi(w_t + \lambda (w_{t+1}-w_t),x_i))\nn\\
&\le f(y_i \Phi(w_t.x_i)) \cdot\exp\left( \sqrt{m}\,aR\ell \lambda h\eta + aLR^2\lambda^2h^2\eta^2\right)\nn
\eea
Noting that this holds for all $i\in[n]$, we deduce that the following holds for the training loss: 
\bea
\max _{v\in[w_t,w_{t+1}]} F(v) &\le \frac{1}{n} \sum_{i=1}^n\max _{v\in[w_t,w_{t+1}]} f(y_i\Phi(v,x_i)) \nn\\
&\le F(w_t) \cdot \exp\left( \sqrt{m}\,aR\ell \lambda h\eta + aLR^2\lambda^2h^2\eta^2\right).\nn
\eea
Recalling that $a\le\frac{1}{m}$ and choosing $C =  \exp(\frac{R\ell \lambda h\eta}{\sqrt{m}} + \frac{LR^2\lambda^2h^2\eta^2}{m})$ leads to \eqref{eq:relax2} and completes the proof.

\subsection{Proof of Lemma \ref{lem:low}}\label{sec:appC}
For the lower bound on the gradient norm, we can write
\bea
\|\nabla F(w)\| &= \frac{1}{n} \| \sum_{i=1}^n f(y_i \Phi(w,x_i))y_i \nabla_1 \Phi (w,x_i)\|\nn
\eea
where $\forall w\in\R^{\widetilde d},x\in\R^d$ the gradient of $\Phi$ with respect to the first argument satisfies the following:
\bea
\nabla_1 \Phi (w,x) = [x a_1 \sigma'(\langle w_1,x\rangle);x a_2 \sigma'(\langle w_2,x\rangle);\cdots;x a_m \sigma'(\langle w_m,x\rangle)] \in \R^{\widetilde d}.\nn
\eea
Equivalently, we can write
\bea
\|\nabla F(w)\| &= \sup_{v\in\R^{\widetilde d}, \|v\|_2=1} \Big\langle \frac{1}{n}  \sum_{i=1}^n f(y_i \Phi(w,x_i))y_i \nabla_1 \Phi (w,x_i),v\Big\rangle\nn
\eea
Choose the candidate vector $v$ as follows
\bea
\bar v=[a_1 w^\star;a_2 w^\star;\cdots;a_m w^\star]\in \R^{\widetilde d}\;\;\;\;\;\ v=\bar v/\|\bar v\|,\nn
\eea 
where $w^\star$ is the max-margin separator that satisfies for all $i\in[n]$, $\frac{y_i  \langle x_i,w^\star\rangle}{\|w^\star\|} \ge \gamma$, where $\gamma$ denotes the margin. We have $\|\bar v\|=\|\tilde a\|\|w^*\|$ where $\tilde a \in\R^m$ is the concatenation of second layer weights $a_j$. Recalling $\sigma'(\cdot)\ge \alpha$, 
\bea
\|\nabla F(w)\| &\ge \frac{1}{\|\tilde a\|\|w^*\|}\frac{1}{n} \sum_{i=1}^n f(y_i \Phi(w,x_i))\cdot y_i  \langle x_i,w^\star\rangle \Big (\sum_{j=1}^m a_j^2\sigma'(\langle w_j,x_i\rangle)\Big)\nn\\
&\ge \|\tilde a\|\frac{\alpha}{n} \sum_{i=1}^n f(y_i \Phi(w,x_i))\cdot \frac{y_i\langle x_i,w^\star\rangle }{\|w^*\|}\nn\\
&\ge  \|\tilde a\|\alpha\cdot (\min_{j\in[n]} \frac{y_j \langle x_j,w^\star\rangle }{\|w^*\|}) \cdot \frac{1}{n} \sum_{i=1}^n f(y_i \Phi(w,x_i))\nn\\
&\ge  \|\tilde a\|\alpha \gamma\cdot F(w).\nn
\eea
This completes the proof of the lemma.

\subsection{Proof of Lemma \ref{lem:low2}}\label{sec:appD}
Recall that,

\bea
\|\nabla F(w)\|_2 &= \sup_{v\in\R^{\widetilde d}, \|v\|_2=1} \Big\langle \frac{1}{n}  \sum_{i=1}^n f(y_i \Phi(w,x_i))y_i \nabla_1 \Phi (w,x_i),v\Big\rangle\nn
\eea
where,
\bea
\nabla_1 \Phi (w,x) = [x a_1 \sigma'(\langle w_1,x\rangle);x a_2 \sigma'(\langle w_2,x\rangle);\cdots;x a_m \sigma'(\langle w_m,x\rangle)] \in \R^{\widetilde d}\nn
\eea
Also, assume $w\in\R^{\widetilde d}$ separates the dataset with margin $ \gamma$, i.e., for all $i\in[n]$
\bea
\frac{y_i\Phi(w,x_i)}{\|w\|} \ge  \gamma.\nn
\eea
choose
\bea
 v=\frac{w}{\|w\|}\nn
\eea 
then
\bea
\|\nabla F(w)\| &\ge \Big\langle \frac{1}{n}  \sum_{i=1}^n f(y_i \Phi(w,x_i))y_i \nabla_1 \Phi (w,x_i),v\Big\rangle \nn\\[4pt]
&= \frac{1}{\|w\|}\frac{1}{n} \sum_{i=1}^n f(y_i \Phi(w,x_i))\cdot y_i \sum_{j=1}^m a_j \langle w_j,x_i\rangle\sigma'(\langle w_j,x_i\rangle)\nn
\eea

Based on the activation function,
\bea\nn
 \langle w_j,x_i\rangle\sigma'(\langle w_j,x\rangle)= \begin{cases}
\ell  \langle w_j,x_i\rangle ~~~  \langle w_j,x_i\rangle\ge 0\\
\alpha  \langle w_j,x_i\rangle~~~~  \langle w_j,x_i\rangle <0.
\end{cases}
\eea
which is equal to $\sigma( \langle w_j,x_i\rangle)$.

Thus, 
\bea
\|\nabla F(w)\|  &\ge  \frac{1}{\|w\|}\frac{1}{n} \sum_{i=1}^n f(y_i \Phi(w,x_i))\cdot y_i \sum_{j=1}^m a_j \sigma(\langle w_j,x_i\rangle)\nn\\
 &=\frac{1}{n} \sum_{i=1}^n f(y_i \Phi(w,x_i))\cdot \frac{y_i\Phi(w,x_i)}{\|w\|}\nn\\
&\ge F(w)\cdot \gamma\nn
\eea
This completes the proof.

\subsection{Proof of Lemma \ref{lem:up-low}}\label{sec:appE}
Recall that,
\bea
F(w):= \frac{1}{n} \sum_{i=1}^n f(y_i \Phi(w,x_i)),\nn\\
\Phi(w,x):= \sum_{j=1}^m a_j \sigma(\langle w_j,x\rangle)\nn
\eea
where $x_i\in\R^d,w_j\in\R^d,a_j\in \R, w=[w_1 w_2 ...w_m]\in \R^{\widetilde d}$. Then noting the exponential nature of the loss function we can write, 
\bea
\|\nabla F(w)\| &= \frac{1}{n} \Big\| \sum_{i=1}^n f'(y_i \Phi(w,x_i))y_i \nabla_1 \Phi (w,x_i)\Big\|\nn\\
&\le \frac{1}{n} \sum_{i=1}^n f(y_i \Phi(w,x_i)\|\nabla_1 \Phi (w,x_i)\|.\nn
\eea
Noting that $\sigma'(\cdot)\le\ell$,
\bea
\|\nabla_1 \Phi (w,x)\|^2 = \sum_{j=1}^m \sum_{i=1}^d (a_j x(i) \sigma'(\langle w_j,x\rangle))^2 \le \frac{\ell^2\|x\|^2}{m} \nn
\eea
Thus $\forall w\in \R^{\widetilde d}$ and $h=\frac{\ell R}{\sqrt{m}}$
\bea
\|\nabla F(w)\| \le h F(w).\nn
\eea
For the Hessian, note that since $|\sigma''(\cdot)|\le L$ and 
\bea
\nabla^2_1 \Phi(w,x) = \frac{1}{m}\text{diag}\left(a_1\sigma '' (\langle w_1,x\rangle) xx^T,\ldots,a_m\sigma''(\langle w_m,x\rangle) xx^T\right),
\eea
then the operator norm of model's Hessian satisfies,
\bea
\|\nabla^2_1 \Phi(w,x)\|^2 \le L^2 R^4 a^2.\nn
\eea
Thus, for the objective's Hessian $\nabla^2 F(w)\in \R^{\widetilde d\times \widetilde d}$, we have
\bea
\|\nabla^2 F(w)\| &= \|\frac{1}{n} \sum_{i=1}^n f(y_i \Phi(w,x_i))y_i \nabla^2_1 \Phi (w,x_i) + f(y_i \Phi(w,x_i)) \nabla_1 \Phi(w,x_i) \nabla_1 \Phi(w,x_i)^\top \|\nn\\
&\le \frac{1}{n}\sum_{i=1}^n f(y_i \Phi(w,x_i)) (\|\nabla^2_1 \Phi (w,x_i) \| + \|\nabla_1 \Phi (w,x_i) \nabla_1\Phi(w,x_i)^\top \|)\nn\\
&=  \frac{1}{n}\sum_{i=1}^n f(y_i \Phi(w,x_i)) (\|\nabla^2_1 \Phi (w,x_i) \| + \|\nabla_1 \Phi (w,x_i)\|_2^2) \nn\\
&\le  (\frac{LR^2}{m^2} + \frac{\ell^2 R^2}{m})F(w). \nn\\\nn
\eea
Denoting $H:= \frac{LR^2}{m^2} + \frac{\ell^2 R^2}{m}$, we have $\|\nabla^2 F(w)\|\le H F(w).$ This concludes the proof. 

\section{Proofs for Section \ref{sec:gen-err}}
\subsection{Proof of Lemma \ref{lem:expansive}}\label{sec:appI}
Define $G(w,v):\R^{\tilde d}\times \R^{\tilde d}\rightarrow \R$ as follows,
\bea
G(w,v) := F(w) - \langle\nabla F(v),w\rangle\nn
\eea
Note that
\bea
\|\nabla^2_{1}G(w,v)\| = \|\nabla^2 F(w)\|\le h F(w).\nn
\eea
Thus by Taylor's expansion of $G$ around its first argument and noting the self-boundedness of Hessian and the convexity of $F$, we have for all $w,\tilde w \in \R^d$,
\bea
G(w,v) &\le G(\tilde w)+\langle \nabla_1 G(\tilde w,v), w-\tilde w\rangle + \frac{1}{2}\max_{v\in[w,\tilde w]} \|\nabla^2 F(v)\| \|w-\tilde w\|^2 \nn\\[4pt]
& \le G(\tilde w)+\langle \nabla_1 G(\tilde w,v), w-\tilde w\rangle +\frac{h}{2}\max_{v\in[w,\tilde w]} F(v) \|w-\tilde w\|^2 \nn\\[4pt]
& \le G(\tilde w)+\langle \nabla_1 G(\tilde w,v), w-\tilde w\rangle +\frac{h}{2}\max (F(w),F(\tilde w)) \|w-\tilde w\|^2. \nn
\eea
Taking minimum of both sides
\bea
\min_{w\in \R^d} G(w,v) &\le \min_{w\in \R^d} G(\tilde w,v)+\langle \nabla_1 G(\tilde w,v), w-\tilde w\rangle +\max (F(w),F(\tilde w)) \frac{h\|w-\tilde w\|^2 }{2} \nn\\[4pt]
&\le G(\tilde w,v) - r\|\nabla_1 G(\tilde w,v)\|^2 + \max (F(\tilde w-r\nabla_1 G(\tilde w,v)), F(\tilde w)) \frac{h r^2 \|\nabla_1 G(\tilde w,v)\|^2}{2}\nn\\[4pt]
&\le G(\tilde w,v) - (r-2r^2 h F(\tilde w) )\|\nabla_1 G(\tilde w,v)\|^2 .\label{eq:co}
\eea
In the second step, we chose $w= \tilde w - r\nabla_1 G(\tilde w,v)$ for a positive constant $r$. Moreover, for the last step we used the following inequality (which we will prove hereafter) that holds under $r\le\frac{1}{h (\max(F(v),F(\tilde w)))}$,

\bea\label{eq:descent}
F(\tilde w-r\nabla_1 G(\tilde w,v))\le 4 F(\tilde w).
\eea

The inequality in \eqref{eq:descent} can be proved according to the following steps. First consider the convexity of $F$ and the self-boundedness of Hessian to derive the Taylor's expansion of $F$ in the following style:
\bea
F(\tilde w-r\nabla_1 G(\tilde w,v)) &= F(\tilde w - r \nabla F(\tilde w)+r\nabla F(v)) \nn\\
&\le F(\tilde w - r \nabla F(\tilde w))+r \langle\nabla F(\tilde w - r \nabla F(\tilde w)),\nabla F(v)\rangle + \frac{h M(w,v)}{2}r^2 \|\nabla F(v)\|^2,\label{eq:smstab}
\eea
where we define,
\bea
M(w,v):= \max(F(\tilde w - r \nabla F(\tilde w)+r\nabla F(v)),F(\tilde w - r \nabla F(\tilde w))).
\eea
We have that if $r\le 1/(h F(\tilde w))$, then
\bea
F(\tilde w - r \nabla F(\tilde w)) \le F(\tilde w)\nn
\eea
Now, suppose that the assumption in \eqref{eq:descent} is false and on the contrary $F(\tilde w-r\nabla_1 G(\tilde w,v))> 4 F(\tilde w)$, then 
\bea
M(w,v)=F(\tilde w-r\nabla_1 G(\tilde w,v)).\nn
\eea
By using Cauchy-Shwarz inequality in \eqref{eq:smstab} together with the self-boundedness properties we deduce that
\bea
F(\tilde w-r\nabla_1 G(\tilde w,v)) &\le F(\tilde w) + r\|\nabla F(\tilde w - r \nabla F(\tilde w))\|\|\nabla F(v)\| + \frac{h r^2}{2}\|\nabla F(v)\|^2 F(\tilde w-r\nabla_1 G(\tilde w,v)) \nn\\[4pt]
& \le F(\tilde w) + r h^2 F(\tilde w - r \nabla F(\tilde w)) F(v) + \frac{r^2 h^3}{2} F^2(v)F(\tilde w-r\nabla_1 G(\tilde w,v))\nn\\[4pt]
& \le F(\tilde w) + r h^2 F(\tilde w) F(v) + \frac{r^2 h^3}{2} F^2(v)F(\tilde w-r\nabla_1 G(\tilde w,v))\nn\\[4pt]
&\le 2 F(\tilde w) + \frac{1}{2}F(\tilde w-r\nabla_1 G(\tilde w,v)),\nn
\eea
The last step is derived by the condition on $r$ and the fact that $h\le1$. The last inequality leads to contradiction. This proves \eqref{eq:descent}. 
Thus, continuing from \eqref{eq:co} and assuming $r\le \frac{1}{2h F(\tilde w)}$
\bea
 F(v) - \langle\nabla F(v),v \rangle \le  F(\tilde w) - \langle\nabla F(v),\tilde w\rangle -\frac{r}{2}\|\nabla F(\tilde w)-\nabla F(v)\|^2\nn
\eea
Exchanging $v$ and $\tilde w$ in the above and noting that under our assumptions it holds that $r\le \frac{1}{2h F(v)}$, we can write
\bea
 F(\tilde w) - \langle\nabla F(\tilde w),\tilde w \rangle \le  F(v) - \langle\nabla F(\tilde w),v\rangle -\frac{r}{2}\|\nabla F(\tilde w)-\nabla F(v)\|^2\nn
\eea
Combining these two together, we end up with the following inequality:
\bea
r\|\nabla F(\tilde w)-\nabla F(v)\| \le \langle\nabla F(v)-\nabla F(\tilde w),v-\tilde w\rangle.\nn
\eea
Therefore $\forall w,v\in \R^d$ if $\eta\le 2r$ (which the RHS itself is smaller than $\frac{1}{h \max(F(v),F(w))}$),
\bea
\|w-\eta\nabla F(w) - (v-\eta\nabla F(v))\|^2 &=\|v-w\|^{2}-2 \eta\langle\nabla F(v)-\nabla F(w), v-w\rangle+\eta^{2}\|\nabla F(v)-\nabla F(w)\|^{2}\nn \\[4pt]
& \leq\|v-w\|^{2}-\left(2 \eta r-\eta^{2}\right)\|\nabla F(v)-\nabla F(w)\|^{2}\nn \\[4pt]
& \leq\|v-w\|^{2}. \nn
\eea
This completes the proof.

\subsection{Proof of Theorem \ref{lem:test}}\label{sec:appG}
Fix $i\in[n]$ and let $w_t^{\neg i}\in\R^{d}$ be the vector obtained at the step $t$ of normalized GD with the following iterations,
\bea\nn
w_{k+1}^{\neg i} = w_k^{\neg i} - \eta_k \nabla F^{\neg i}(w_k^{\neg i}),  
\eea
where $\eta_{k}$ denotes the step-size at step $k$ which satisfies $\eta_{k}\le \frac{1}{h F^{\neg i}(w_k^{\neg i})}$ for all $k\in[t-1]$. Also, we define the leave-one-out training loss for $i\in[n]$ as follows:
\bea
F^{\neg i}(w): = \frac{1}{n} \sum_{\substack{{j=1}\\ {j\neq i}}}^n f(w,z_j). \nn
\eea
In words, $w_t^{\neg i}$ is the output of normalized GD at iteration $t$ when the $i$th sample is left out while the step-size is chosen independent of the $i$ th sample. Thus, we can write
\bea
\E[\widetilde F(w_t)- F(w_t)] &= \frac{1}{n}\sum_{i=1}^n \E[f(w_t,z)-f(w_t^{\neg i},z)] + \frac{1}{n}\sum_{i=1}^n \E[f(w_t^{\neg i},z_i)-f(w_t,z_i)]\nn\\
&\le \frac{2G}{n} \sum_{i=1}^n\E[ \|w_t-w_t^{\neg i}\|]\label{eq:stability-ngd}
\eea
Since the loss function is non-negative, $F^{\neg i}(w_t) \le F(w_t)$ for all $i$. Thus, by assumption of the theorem the step-size satisfies $\eta_t\le\frac{1}{h \delta F(w_t)}\le\frac{1}{h \delta F^{\neg i}(w_t)}, \; \forall i\in[n]$. By the definition of $\delta$, this choice of step-size guarantees that $\eta_t\le \frac{1}{h F^{\neg i}(w_t^{\neg i})}$. Recalling that $\delta\ge 1$, we deduce that $\eta_t\le\frac{1}{h \max(F^{\neg i}(w_t),F^{\neg i}(w_t^{\neg i}))}$, which allows us to apply Lemma \ref{lem:expansive}. In particular, by unrolling $w_{t+1}$ and $w_{t+1}^{\neg i}$, and using our result from Lemma \ref{lem:expansive} on the non-expansiveness of normalized GD we can write,
\bea
\Big\|w_{t+1}-w_{t+1}^{\neg i}\Big\| &=\Big\|w_{t}-\frac{1}{n}\eta_t\sum_{j=1}^n \nabla f(w_{t},z_j) - w_{t}^{\neg i}+ \frac{1}{n}\eta_t \sum_{j\neq i}^n \nabla f(w_{t}^{\neg i},z_j) \Big\| \nn\\
& = \Big\|w_{t}-\eta_t \nabla F^{\neg i}(w_t) -\frac{1}{n}\eta_t \nabla f(w_t,z_i)- w_{t}^{\neg i} + \eta_t \nabla F^{\neg i}(w_t^{\neg i}) \Big\|\nn \\
&\le \Big\|w_{t}-\eta_t \nabla F^{\neg i}(w_t) - w_{t}^{\neg i} + \eta_t \nabla F^{\neg i}(w_t^{\neg i}) \Big\|  + \frac{1}{n}\eta_t\|\nabla f(w_t,z_i)\|\nn \\
& \leq \Big\|w_{t}-w_{t}^{\neg i}\Big\|+\frac{1}{n}\eta_t\Big\|\nabla f\left(w_{t}, z_{i}\right)\Big\|\nn\\
& \leq \Big\|w_{t}-w_{t}^{\neg i}\Big\|+\frac{1}{n}h\eta_t f(w_{t},z_i).\label{eq:103}
\eea
This result holds for all $i\in[n]$. By averaging over all training samples,
\bea
\frac{1}{n}\sum_{i=1}^n \| w_{t+1}-w_{t+1}^{\neg i}\| \le \frac{1}{n}\sum_{i=1}^n \|w_{t}-w_{t}^i\| + \frac{h}{n}\eta_t F(w_t).\nn
\eea
Thus, by telescoping sum over $t$, for the last iteration we have,
\bea
\frac{1}{n}\sum_{i=1}^n \| w_{_T}-w_{_T}^{\neg i}\| \le \frac{h}{n}\sum_{t=0}^{T-1} \eta_t F(w_t)\nn
\eea
Next, we recall \eqref{eq:stability-ngd} which allows us to bound the generalization gap,
\bea
\E[\widetilde F(w_{_T})- F(w_{_T})]&\le \frac{2G h}{n} \sum_{t=0}^{T-1} \eta_t F(w_t)\nn\\
&\le \frac{2G T}{n}.\nn
\eea
This completes the poof for $L$- Lipschitz losses.

For $\tilde L$-smooth losses, the following relation holds between test and train loss and the leave-one-out distance (e.g., see \cite[Lemma 7]{schliserman2022stability}, \cite[Theorem2]{lei2020fine}):
\begin{align}
\E[\widetilde F(w)]\le 4 \E[ F( w)] + \frac{3\tilde L^2}{n} \sum_{i=1}^n \E[\|w-w^{\neg i}\|^2].\label{eq:stab_smooth}
\end{align}
Note the dependence on $\|w-w^{\neg i}\|^2$. Recalling \eqref{eq:103}, we had
\bea
\left\|w_{t+1}-w_{t+1}^{\neg i}\right\| \leq \left\|w_{t}-w_{t}^{\neg i}\right\|+\frac{1}{n}\eta_t\,h\, f(w_{t},z_i)\nn
\eea
By telescoping summation,
\bea
\| w_{_T} - w_{_T}^{\neg i}\| \le \frac{h}{n} \sum_{t=0}^{T-1} \eta_t f(w_t.z_i)\nn
\eea
this gives the following upper bound on the averaged squared norm, 
\bea
\frac{1}{n} \sum_{i=1}^n \| w_{_T} - w_{_T}^{\neg i}\|^2 &\le \frac{h^2}{n^3}\sum_{i=1}^{n}(\sum_{t=1}^{T-1} \eta_t f(w_t.z_i))^2\nn\\[3pt]
&\le  \frac{h^2}{n^3}(\sum_{i=1}^{n}\sum_{t=0}^{T-1} \eta_t f(w_t.z_i))^2\nn\\[3pt]
&= \frac{h^2}{n} (\sum_{t=0}^{T-1} \frac{\eta_t}{n}\sum_{i=1}^n f(w_t.z_i))^2\nn\\[3pt]
&= \frac{h^2}{n} (\sum_{t=0}^{T-1} \eta_t F(w_t))^2.\nn
\eea
Hence, replacing these back in \eqref{eq:stab_smooth}, 
\bea
\E[\widetilde F(w_{_T})] &\le 4 \E[ F( w_{_T})] + \frac{3\tilde L^2 h^2}{n} (\sum_{t=0}^{T-1} \eta_t F(w_t))^2 \nn\\
&\le 4 \E[ F( w_{_T})] + \frac{3\tilde L^2}{n} T.\nn
\eea
%
This gives the desired result for $\tilde L$-smooth losses in part (ii) of the lemma and completes the proof. 

\subsection{On $\delta$ in Theorem \ref{lem:test}}\label{sec:appc3}
\begin{lemma}\label{lem:onthm8}
Assume the iterates of normalized GD with $\eta\le 1/h$, zero initialization (w.l.o.g) and $m=\beta T^2$ hidden neurons for any constant $\beta>0$. Then $\delta$ in the statement of Theorem \ref{lem:test} is satisfied with $\delta = \exp(\frac{2R\ell}{\sqrt{\beta}} + \frac{4LR^2}{\beta}).$
\end{lemma}
\begin{proof}
By the log-Lipschitzness property in \eqref{eq:conv-sem} and recalling $a=1/m$,
\bea\nn
F^{\neg i}(w_T^{\neg i}) &\le F^{\neg i}(w_T)\cdot \exp\Big( \frac{R\ell}{\sqrt{m}} \| w_T^{\neg i}-w_T\| + \frac{LR^2}{m}\|w_T^{\neg i}-w_T\|^2 \Big)\\
&\le F^{\neg i}(w_T)\cdot \exp\Big( \frac{R\ell}{\sqrt{m}} (\| w_T^{\neg i}\|+\|w_T\|) + \frac{2LR^2}{m}(\|w_T^{\neg i}\|^2+\|w_T\|^2) \Big).\label{eq:propo10}
\eea
Now we note that the weight-norm can be upper bounded as following:
\begin{align*}
\|w_T\| &= \Big\|w_{T-1}-\frac{\eta}{F(w_{T-1})}\nabla F(w_{T-1})\Big\|\\
& = \Big\|w_0 - \eta \sum_{t=0}^{T-1} \frac{\nabla F(w_t)}{F(w_t)}\Big\|\\
&\le \eta \sum_{t=0}^{T-1}\Big\|\frac{\nabla F(w_t)}{F(w_t)}\Big\|\\
&\le \eta h T.
\end{align*}
Similarly, we can show that $\|w_T^{\neg i}\|\le \eta h T$. Therefore by $m=\beta T^2$ and \eqref{eq:propo10},
\begin{align*}
  F^{\neg i}(w_T^{\neg i}) &\le  F^{\neg i}(w_T)\cdot \exp\Big( \frac{R\ell}{\sqrt{m}} (\| w_T^{\neg i}\|+\|w_T\|) + \frac{2LR^2}{m}(\|w_T^{\neg i}\|^2+\|w_T\|^2) \Big)\\
  &\le F^{\neg i}(w_T)\cdot\exp\Big( \frac{2R\ell}{\sqrt{m}} (\eta h T) + \frac{4LR^2}{m} \eta^2 h^2 T^2 \Big)\\
  &\le  F^{\neg i}(w_T)\cdot\exp\Big(\frac{2R\ell}{\sqrt{\beta}} + \frac{4LR^2}{\beta}\Big),
\end{align*}
where the last step follows by $\eta h\le 1$ as per assumptions on the step-size. This completes the proof.
\end{proof}
\subsection{Proof of Corollary \ref{lem:teste}}\label{sec:appH}
First, note that if $F(w)< \delta\le 1$, then $\|w\| \ge \frac{1}{\ell R} (\log (\frac{1}{2\delta})-\sigma_0)$,  where $\sigma_0=|\sigma(0)|$, since if the lower-bound on $\|w\|$ is incorrect then, 
\bea
F(w)&= \frac{1}{n} \sum_{i=1}^n \log(1+\exp(-y_i \Phi(w,x_i))) \nn\\
&\ge \frac{1}{n}\sum_{i=1}^n \log(1+\exp(-\ell\|w\|\|x_i\|-\sigma_0)) \nn\\
&\ge \frac{1}{n} \sum_{i=1}^n \log(1+\exp(\log(2\delta)))\nn\\
&\ge \delta,\nn
\eea
In the final step, we made use of the inequality $\log(1+2\delta)\geq\delta$ for $\delta\leq 1$. Additionally, the validity of the second step relies on the Lipschitz property of the model, as demonstrated below.
\bea
y\Phi(w,x)&= \sum_{j=1}^m ya_j \sigma(\langle w_j, x\rangle) \nn\\
&\le \sum_{j=1}^m |a_j|\cdot |\sigma(\langle w_j, x\rangle)| \nn\\
&\le \sum_{j=1}^m |a_j| (\sigma_0+\ell|\langle w_j,x\rangle|)\nn\\
&\le \sigma_0 ||\tilde a||_1 + \ell\|x\|_2\sum_{j=1}^m |a_j|\cdot \|w_j\| \nn\\[4pt]
&\le \sigma_0 \|\tilde a\|_1 + \ell\|x\|_2\|\tilde a\|_2\|w\|_2 \nn
\eea
This is true due to $\ell$-Lipschitz activation and our assumption that $\|\tilde a\|_1\le m \|\tilde a\|_{\infty} = 1$, where $\tilde a\in\R^m$ is the concatenation of second layer weights. 

Now, note that due to the convergence of training loss there exists a $\tau>0$ such that at iteration $t$ the following holds:
\bea
F(w_t) \le (1-\tau)^t F(w_0).\nn
\eea
Hence the weight's norm at iteration $t$ satisfies,
\bea
\| w_t\| \ge  \frac{t}{R} \log(\frac{1}{2-2\tau}) - \frac{\sigma_0}{R} = \Theta (t).
\eea
For the test error, by defining $\mathcal{F}$ to be the set of data points labeled incorrectly by $\Phi(w_t,\cdot)$, we can write
\bea
\E_{(x,y)\sim \mathcal D}[f(y \Phi(w_t,x))] &= \lim_{n\rightarrow\infty} \frac{1}{n}\sum_{i=1}^n f(y_i \Phi(w_t,x_i))\nn \\
&\ge \lim_{n\rightarrow\infty} \frac{1}{n} \sum_{i\in\mathcal{F}} f(y_i \Phi(w_t,x_i))\nn\\
& = \lim_{n\rightarrow\infty} \frac{1}{n} \sum_{i\in \mathcal{F}} f(-|\Phi(w_t,x_i)|)\nn\\
& =\lim_{n\rightarrow\infty} \frac{1}{n} \sum_{i\in \mathcal{F}} \log(1+\exp(|\Phi(w_t,x_i)|))\nn\\
&\ge \frac{1}{3}\|w_t\|\cdot \lim_{n\rightarrow\infty} \frac{1}{n} \sum_{i\in \mathcal{F}}  \frac{|\Phi( w_t,x_i)|}{\|w_t\|}\nn\\
& \ge \frac{1}{3} \gamma \|w_t\| \E_{(x,y)\sim\mathcal{D}}[\mathbb{I}({\textsc{sign}(\Phi(w_t,x))\neq y})]\nn\\[4pt]
& = \Theta (t) \E_{(x,y)\sim\mathcal{D}}[\mathbb{I}({\textsc{sign}(\Phi(w_t,x))\neq y})]\nn
\eea

Where we used the fact that $\log(1+\exp(t))\ge \frac{1}{3}t$ and the one to the last line inequality is due to Assumption \ref{ass:margin} i.e., $ \frac{| \Phi(w_t,x_i)|}{\|w_t\|} \ge \gamma$ with high probability over $(x_i,y_i)\overset{iid} {\sim}\mathcal{D}$.
Hence the test error satisfies, 
\bea
\E[\mathbb{I}(y\neq\textsc{sign}(\Phi(w_t,x)))] = O(\frac{F(w_t)}{t}). \nn
\eea
This together with the test loss bound in Theorem \ref{lem:test} yields the statement of the corollary and completes the proof. 
\\
\\
\section{Gradient Flow}\label{sec:appJ}
\begin{proposition}[Normalized GD in continuous time]\label{propo:grad_flow}
Let the loss function $F$ satisfy self-lower boundedness of the gradient with parameter $\mu$ (Definition \ref{ass:self-lower}) and the self-bounded gradient property with parameter $h$ (Definition \ref{ass:self-bounded}). Consider normalized gradient descent with the Gradient flow differential equation given by $\frac{d}{dt} w_t= -\nabla F(w_t)/F(w_t)$. Then the training loss  at time $T$ satisfies 
\bea
F(w_0)\cdot \exp(-h^2 T)\le F(w_{_T}) \le F(w_0)\cdot \exp(-\mu^2 T). \nn
\eea
\end{proposition}
\begin{proof}
Based on the assumptions, we have
\bea
\dot{w}_t:= \frac{d}{dt} w_t= -\frac{\nabla F(w_t)}{F(w_t)}.\nn
\eea
Then,
\bea
\frac{d}{dt} F(w_t) = \nabla F(w_t)^\top \dot{w}_t =- \frac{\|\nabla F(w_t)\|^2}{F(w_t)} \nn
\eea
By self-lower bounded property we have $\frac{d}{dt} F(w_t)  \le -\mu^2 F(w_t)$. Thus,
\bea
\frac{d}{dt} \log(F(w_t)) = \frac{\frac{d}{dt} F(w_t)}{F(w_t)} \le -\mu^2.\nn
\eea
By integrating from $t=0$ to $t=T$ one can deduce that,
\bea
\log(F(w_{_T})) - \log(F(w_0)) \le -\mu^2 T.\nn
\eea
This leads to the desired upper-bound for $F(w_{_T})$. A similar approach by using the self-bounded gradient property leads to the lower bound. This concludes the proof. 
\end{proof}

\section{Proofs for Section \ref{sec:stoch}}
\subsection{On the Strong Growth Condition}\label{sec:appK}
\begin{proposition}
Under the self-bounded gradient property (Definitions \ref{ass:self-lower}-\ref{ass:self-bounded}) there exists a $\rho$ such that the strong growth condition is satisfied i.e.,
\bea\nn
\E_{z}[\|\nabla F_z(w)\|^2]\le \rho \|\nabla F(w)\|^2.
\eea
\end{proposition}
\begin{proof}
By the self-bounded gradient property and noting the non-negativity of $f$ we have,
\bea
\E_{z}[\|\nabla F_z(w)\|^2] &\le h^2 \E[(F_z(w))^2]\nn \\[4pt]
&\le h^2 n (F(w))^2\nn\\[4pt]
&\le \frac{h^2n}{\mu^2} \|\nabla F(w)\|^2.\nn
\eea
This completes the proof. 
\end{proof}
\subsection{Proof of Theorem \ref{thm:SNGD}}\label{sec:appe2}
Following the proof of Theorem \ref{lem:train} and noting the log-Lipschitzness and the self-bounded Hessian property we derive that,
\bea
F(w_{t+1}) &\le F(w_t) + \langle \nabla F(w_t),w_{t+1}-w_t\rangle + \frac{1}{2}HC\,F(w_t)\, \|w_{t+1}-w_t\|^2\nn\\[4pt]
&= F(w_t) - \eta_t  \langle \nabla F(w_t), \nabla F_{z_t}(w_t)\rangle + \frac{1}{2}HC \eta_t^2 F(w_t) \|\nabla F_{z_t}(w_t)\|^2\label{eq:sgdsmo}
\eea
Taking expectation with respect to $z_t$ and using self-boundedness property yields, 
\bea
\E_{z_t}[F(w_{t+1})] &\le F(w_t) - \eta_t \|\nabla F(w_t)\|^2 + \frac{1}{2}HC\eta_t^2 F(w_t)\E_{z_t}[\|\nabla F_{z_t}(w_t)\|^2]\nn\\[4pt]
&\le F(w_t) - \eta_t \|\nabla F(w_t)\|^2 + \frac{1}{2} \rho HC\eta_t^2 F(w_t) \|\nabla F(w_t)\|^2\nn\\[4pt]
&\le F(w_t) - \mu^2\eta_t (F(w_t))^2+\frac{1}{2}\rho H h^2C\eta_t^2 (F(w_t))^3\nn
\eea
Let $\eta_t = \frac{\eta}{F(w_t)}$, since $\eta \le \frac{\mu^2}{HC\rho h^2}$
\bea
\E_{z_t}[F(w_{t+1})]&\le F(w_t) (1-\eta\mu^2 + \frac{1}{2}\rho Hh^2C\eta^2)\nn\\[4pt]
&\le (1-\frac{\eta\mu^2}{2}) F(w_t).\nn
\eea
This completes the proof. 
\section{Experiments on stochastic normalized GD}\label{sec:appnum}
\begin{figure}[H]
\centering
\includegraphics[width=5.7cm,height=4.3cm]{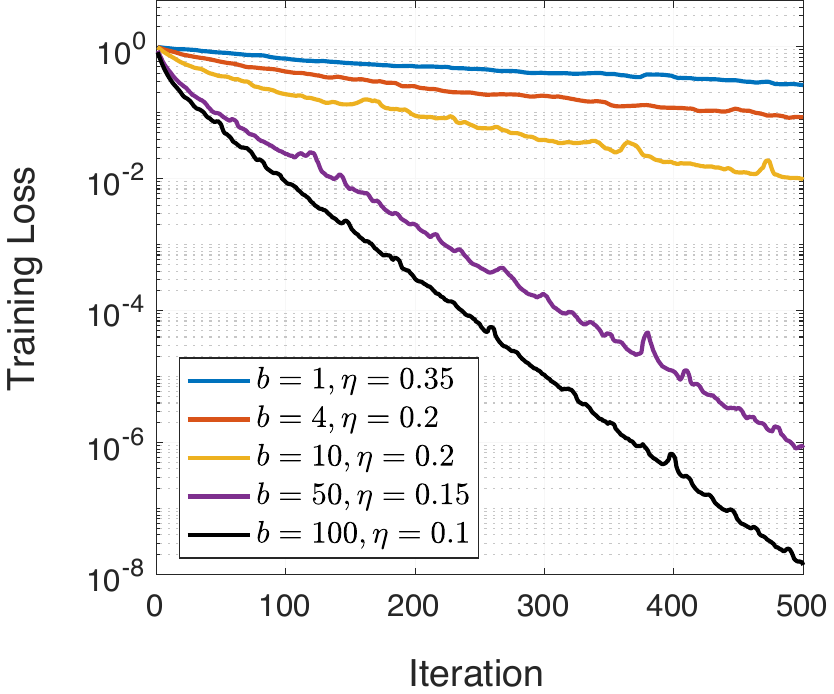}
\hspace{1in}
\includegraphics[width=5.7cm,height=4.5cm]{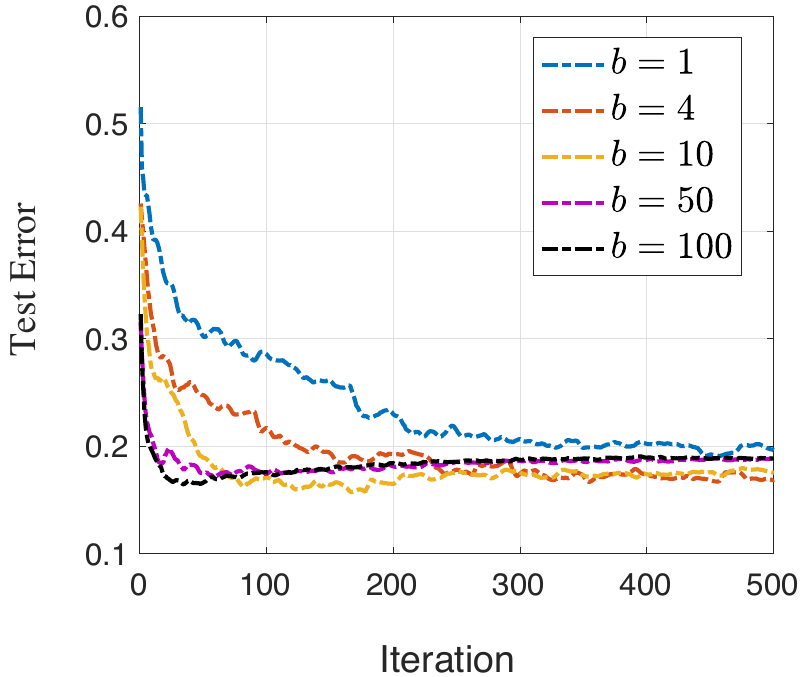}
\\[10pt]
\;\;\;\;\;\;\includegraphics[width=5.2cm,height=4.2cm]{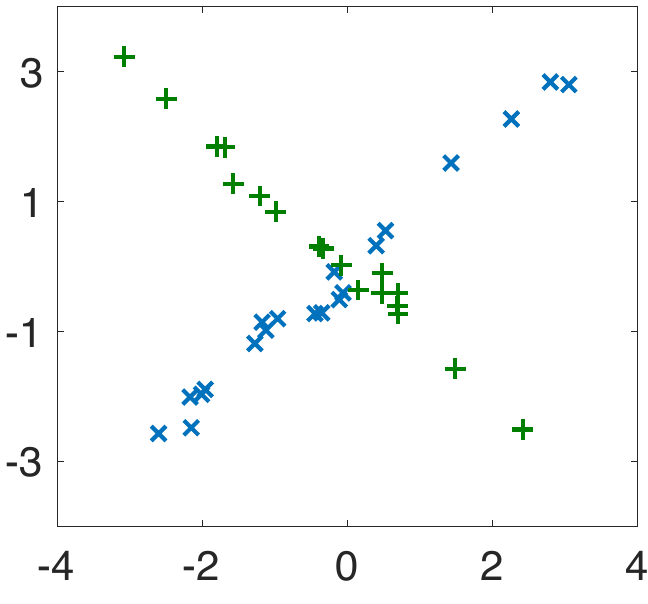}
\hspace{0.9in}
\includegraphics[width=6cm,height=4.5cm]{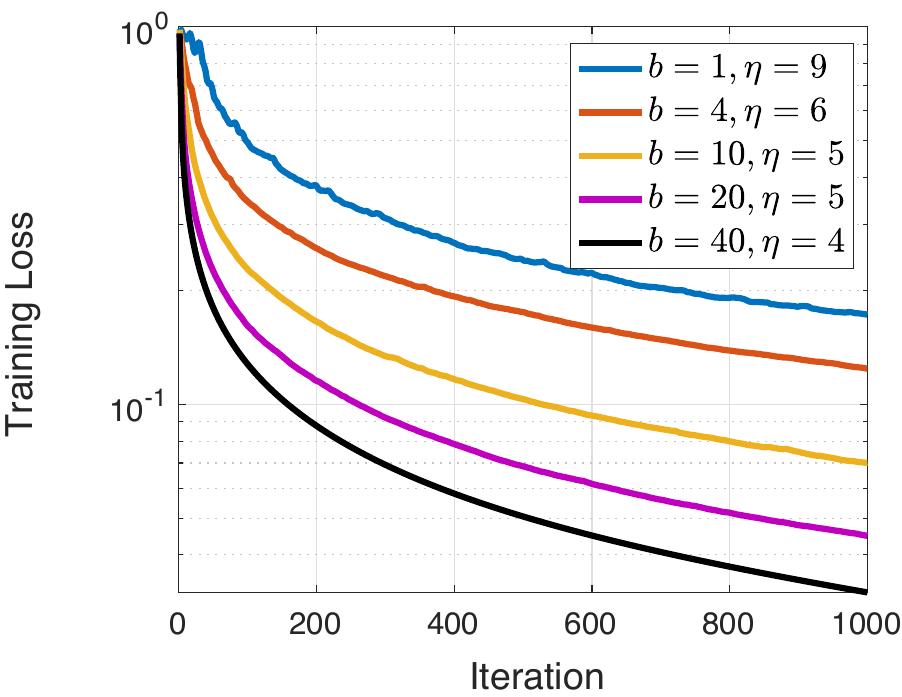}
\caption{(Top) Training loss and Test error of stochastic normalized GD (Eq.\eqref{eq:stocngd}) on linear classification with signed measurements $y=\sign(x^\top w^\star)$ with $d=50,\,n=100$. Here $`b`$ denotes the batch-size and $`\eta`$ is the fine-tuned step-size. (Bottom) Training loss of stochastic normalized GD on the dataset depicted in the left figure ($d=2,n=40$) for a two-layer neural network with $m=50$ hidden neurons.}
\label{fig:3}
\end{figure}
In this section, we evaluate the performance of stochastic normalized GD in Eq.\eqref{eq:stocngd} for linear and non-linear models. In Figure \ref{fig:3} (Top), we consider binary linear classification on signed data with the exponential loss and plot the training loss and test error performance based on iteration number. $b$ denotes the batch-size from the sample dataset size of $n=100.$ The weight vector is initialized at zero for all curves ($w_0=0_d$). The right plot shows the test error for the same setup, where the optimal test error ($\tilde F_{0-1}^\star\approx0.17$) is reached at various iteration numbers for each batch-size. In particular, for $b=10$(yellow line) stochastic normalized GD achieves the final test accuracy at almost the same time as the full-batch normalized GD (black line) while using $1/10$ th gradient computations.  Figure \ref{fig:3} (Bottom) depicts the synthetic dataset of size $n=40$ in $\R^2$ alongside with the training loss performance for each choice of batch-size $b$. Here we used a leaky-ReLU activation function as in Eq.\eqref{eq:piecewisel} with $\ell=1,\alpha=0.2$.

\end{document}